\documentclass[10pt,twocolumn,letterpaper]{article}

\usepackage{cvpr}              

\usepackage[dvipsnames]{xcolor}

\definecolor{cvprblue}{rgb}{0.21,0.49,0.74}
\usepackage[utf8]{inputenc} 
\usepackage[accsupp]{axessibility}
\usepackage[T1]{fontenc}    
\usepackage{url}            
\usepackage{booktabs}       
\usepackage{multirow}
\usepackage{amsfonts}       
\usepackage{nicefrac}       
\usepackage{microtype}      
\usepackage{xcolor}         
\usepackage{amsmath}
\usepackage{amssymb}
\usepackage{times}
\usepackage{epsfig}
\usepackage{graphicx}
\usepackage{subfloat}
\usepackage{xspace}
\usepackage{tikz}
\usepackage{bbm}
\usepackage{colortbl,booktabs}
\usepackage{makecell}
\usepackage{diagbox} 
\usepackage{multicol}
\usepackage{array}
\usepackage{amsmath, bm}
\usepackage{amsthm}
\usepackage{dsfont}
\usepackage{amssymb}
\usepackage{algorithm}
\usepackage{algpseudocode}
\usepackage{varwidth}
\usepackage{pifont}
\usepackage{makecell}
\usepackage{wrapfig}
\usepackage{enumitem}
\usepackage{caption}
\usepackage{listings}
\usepackage{color}
\usepackage{tikz}
\usetikzlibrary{arrows.meta}
\usepackage{subcaption, multirow, overpic, textpos}

\makeatletter
\DeclareRobustCommand\onedot{\futurelet\@let@token\@onedot}
\def\@onedot{\ifx\@let@token.\else.\null\fi\xspace}

\def\eg{\emph{e.g}\onedot} 
\def\ie{\emph{i.e}\onedot} 
 
\def\etc{\emph{etc}\onedot} 
\def\wrt{w.r.t\onedot} 

\makeatother
\newlength\savewidth\newcommand\shline{\noalign{\global\savewidth\arrayrulewidth
  \global\arrayrulewidth 1pt}\hline\noalign{\global\arrayrulewidth\savewidth}}
\newcommand{\tablestyle}[2]{\setlength{\tabcolsep}{#1}\renewcommand{\arraystretch}{#2}\centering\footnotesize}
\renewcommand{\paragraph}[1]{\vspace{1.25mm}\noindent\textbf{#1}}

\newcolumntype{x}[1]{>{\centering\arraybackslash}p{#1pt}}
\newcolumntype{y}[1]{>{\raggedright\arraybackslash}p{#1pt}}
\newcolumntype{z}[1]{>{\raggedleft\arraybackslash}p{#1pt}}

\newcommand{\app}{\raise.17ex\hbox{$\scriptstyle\sim$}}

\definecolor{deemph}{gray}{0.6}

\definecolor{baselinecolor}{gray}{.9}
\newcommand{\baseline}[1]{\cellcolor{baselinecolor}{#1}}

\usepackage[pagebackref,breaklinks,colorlinks,citecolor=cvprblue]{hyperref}

\def\paperID{3657} 
\def\confName{CVPR}
\def\confYear{2024}

\title{Once for Both: Single Stage of Importance 
and Sparsity Search for Vision Transformer Compression}

\author{Hancheng Ye$^{1, 2}$ \quad
    Chong Yu$^{3}$ \quad
    Peng Ye$^{1}$ \quad
    Renqiu Xia$^{2, 4}$ \quad
    Yansong Tang$^{5}$
    \\
    Jiwen Lu$^{5}$ \quad
    Tao Chen$^{1}$\footnotemark[2] \quad
    Bo Zhang$^{2}$\footnotemark[3]
    \\
    {\normalsize $^{1}$School of Information Science and Technology, Fudan University} \quad
    {\normalsize $^{2}$Shanghai Artificial Intelligence Laboratory} \\
    {\normalsize $^{3}$Academy for Engineering and Technology, Fudan University}\quad
    {\normalsize $^{4}$Shanghai Jiao Tong University}\quad
    {\normalsize $^{5}$Tsinghua University}
}

\begin{document}
\maketitle
\renewcommand{\thefootnote}{\fnsymbol{footnote}}
\footnotetext[2]{Corresponding author. 
$^\ddagger$Project lead.}
\begin{abstract}
Recent Vision Transformer Compression (VTC) works mainly follow a two-stage scheme, where the importance score of each model unit is first evaluated or preset in each submodule, followed by the sparsity score evaluation according to the target sparsity constraint. Such a separate evaluation process induces the gap between importance and sparsity score distributions, thus causing high search costs for VTC.
In this work, for the first time, we investigate how to integrate the evaluations of importance and sparsity scores into a \textbf{single} stage, searching the optimal subnets in an efficient manner. Specifically, we present OFB, a cost-efficient approach that simultaneously evaluates both importance and sparsity scores, termed Once for Both (OFB), for VTC. First, a bi-mask scheme is developed by entangling the importance score and the differentiable sparsity score to jointly determine the pruning potential (prunability) of each unit.
Such a bi-mask search strategy is further used together with a proposed adaptive one-hot loss to realize the progressive-and-efficient search for the most important subnet. Finally, Progressive Masked Image Modeling (PMIM) is proposed to regularize the feature space to be more representative during the search process, which may be degraded by the dimension reduction.
Extensive experiments demonstrate that OFB can achieve superior compression performance over state-of-the-art searching-based and pruning-based methods under various Vision Transformer architectures, meanwhile promoting search efficiency significantly, \eg, costing one GPU search day for the compression of DeiT-S on ImageNet-1K.
\end{abstract}
\section{Introduction}
\label{sec:intro}

\begin{figure}[t]
\vspace{-1em}
    \centering
    \small
    \resizebox{0.93\linewidth}{!}{
    \includegraphics{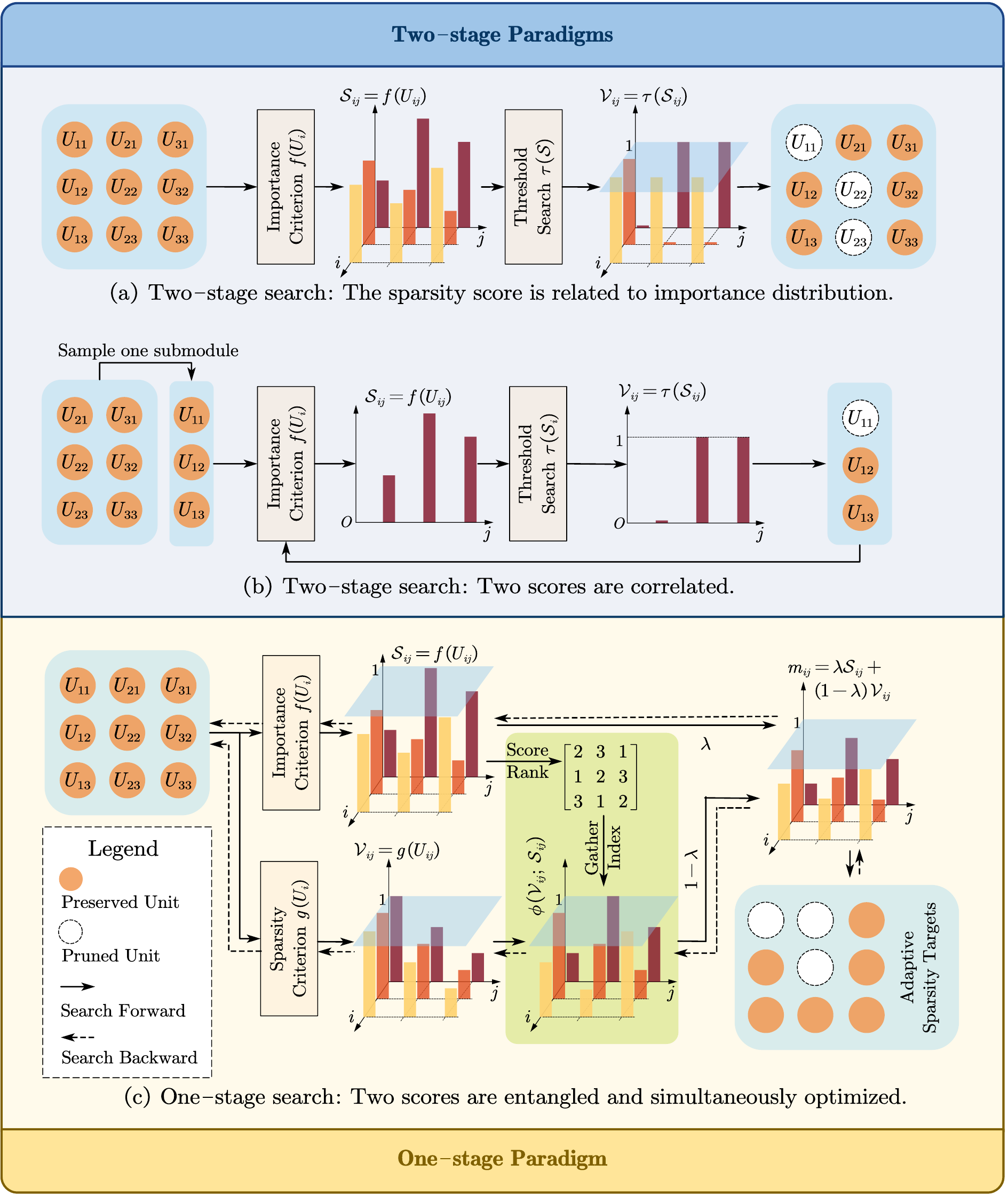}}
    \vspace{-1em}
    \caption{The relationship between importance and sparsity score distributions in different search paradigms. (a) Importance scores are fixed during sparsity search, and sparsity scores are related to importance scores. (b) Importance scores of one submodule are also related to the sparsity of other to-prune submodules. (c) Importance and sparsity scores are entangled and simultaneously optimized, thus correlated at forward and backward phases of searching.}
    \label{fig:correlation}
    \vspace{-2.6em}
\end{figure}

Vision Transformers (ViTs) are developing rapidly in many practical tasks, but they suffer from substantial computational costs and storage overhead, preventing their deployments for resource-constrained scenarios. Vision Transformer Compression (VTC), as an effective technique to relieve such problems, has advanced a lot and can be divided into several types including Transformer Architecture Search (TAS) \cite{autoformer, sss, vitas, uninet, nasvit, tu2023efficient, tftas} and Transformer Pruning (TP)~\cite{spvit, iared, vit_slim, s2vit, s2vite, cpvit, vtclfc, wdpruning} paradigms. Although both TAS and TP can produce compact ViTs, their search process for the target sparsity often relies on a two-stage scheme, \ie, \textit{importance-then-sparsity evaluation}\footnote{The importance evaluation aims at learning each unit's contribution to the prediction performance, while the sparsity evaluation aims at learning each unit's pruning choice. In general, the importance and sparsity score distributions are correlated in the search process, as shown in Fig. \ref{fig:correlation}.} for units (\eg, filters) in each dimension / submodule, which potentially hinders search performance and efficiency of both paradigms.

\begin{figure}[t]
    \centering
    \small
    \resizebox{0.88\linewidth}{!}{
    \includegraphics{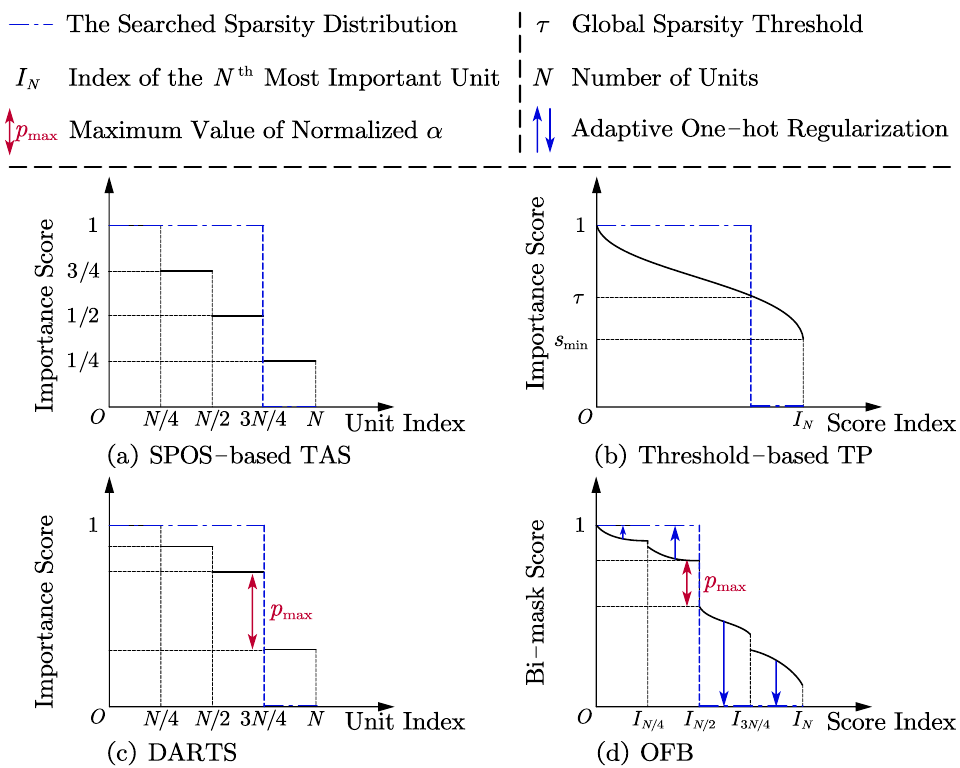}}
    \vspace{-1em}
    \caption{Different paradigms for VTC. (a): SPOS-based TAS implicitly encodes the piecewise-decreasing importance scores for units due to the uniform sampling in pre-training; (b): The threshold-based TP explicitly evaluates the importance scores for units and sets a global threshold to perform pruning; (c): DARTS learns the importance distribution in a differentiable manner and selects the subnet of the highest architecture score; (d): OFB proposes the bi-mask score that entangles importance and sparsity scores together, to perform the search process in a \textbf{single} stage.}
    \label{fig:difference}
    \vspace{-2em}
\end{figure}

As for TAS that mainly follows the Single-Path-One-Shot (SPOS) search paradigm \cite{spos}, the importance scores of units in each submodule are implicitly encoded into the supernet \cite{vitas}, as shown in Fig. \ref{fig:difference}\textcolor{red}{a}. This is mainly due to the ordinal weight-sharing mechanism during the pre-training of the pre-designed supernet \cite{vitas}.
In other words, the submodules with small indexes are implicitly assigned higher importance scores by uniform sampling during the supernet pre-training. Afterwards, evolutionary algorithms are employed to search for the optimal subnet given the implicitly-encoded importance score distribution and target sparsity constraint \cite{autoformer, nat, bossnas, vitas, tftas}. Such an implicit encoding process causes TAS limited to compressing a supernet from scratch, thus leading to a high search cost.

On the other hand, for TP that adopts the threshold-based pruning paradigm, the importance scores are pre-evaluated by a designed criterion, followed by the sparsity search using a designed strategy based on the importance distribution. However, searching for the fine-grained discrete sparsity from the evaluated importance distribution of each dimension is intractable and identified as an NP-hard problem \cite{rbp}. As visualized in Fig. \ref{fig:difference}\textcolor{red}{b}, the importance score distribution of one dimension is usually continuous, with most points distributed around the mean value. Considering that the importance distribution varies in different dimensions and may be influenced by the pruning choice of other dimensions \cite{vitas, pagcp}, the traditional threshold-based methods can hardly search for the optimal compact models in a global manner. From the above analysis, \textit{the high compression costs can be largely attributed to the separate score evaluation, and the gap between importance and sparsity score distributions}.

To tackle the above issues induced by the two-stage VTC scheme, we propose to conduct the ViTs search in a one-stage manner, where the importance and sparsity scores are learned simultaneously and entangled, to learn a discrete sparsity distribution from the entangled distribution adaptively. To achieve this, inspired by the differentiable search strategy in DARTS \cite{darts, pdarts, ye2022beta}, we relax the sparsity score to a differentiable variable, and formulate a bi-mask score that entangles the importance and sparsity scores of each unit, to jointly assess the unit's prunability. Secondly, to optimize the bi-mask score, we introduce an adaptive one-hot loss function to adaptively convert the continuous bi-mask score into a binary one, \ie, the unit's pruning choice, as shown in Fig. \ref{fig:difference}\textcolor{red}{d}. Finally, during the search, we further develop a Progressive Masked Image Modeling (PMIM) technique, to regularize the dimension-reduced feature space with negligible additional costs. Our main contributions are:

\begin{itemize}
    \item To our best knowledge, our method is the first to explore the entanglement of importance and sparsity distributions in VTC, which relieves the bottleneck of searching the discrete sparsity distribution from the continuous importance distribution, highlighting the search efficacy and effectiveness of various ViTs compression.
    \item We develop a novel one-stage search paradigm containing a bi-mask weight-sharing scheme and an adaptive one-hot loss function, to simultaneously learn the importance and sparsity scores and determine the units' prunability. Moreover, a PMIM regularization strategy is specially designed during searching, which gradually intensifies the regularization for representation learning as the feature dimension continues to be reduced.
    \item Extensive experiments are conducted on ImageNet for various ViTs. Results show that OFB outperforms existing TAS and TP methods with higher sparsity and accuracy, and significantly improves search efficiency, \eg, costing one GPU search day to compress DeiT-S on ImageNet.
\end{itemize}

\section{Related Works}
\vspace{-0.5em}
\label{sec:related_work}

\paragraph{Transformer Architecture Search.}
Recently, with various Vision Transformers spawning \cite{vit, swin, deit, twins}, several works have explored searching for the optimal Transformer-based architecture. Existing Transformer Architecture Search (TAS) works \cite{autoformer, vitas, tftas} mainly follow the SPOS NAS \cite{spos} scheme, which first trains the supernet from scratch by iteratively training the sampled subnets, then searches for the target optimal subnet, followed by retraining the searched model. These methods focus on either designing the search space or the training strategy for the randomly initialized supernet, yet the supernet training is still time-consuming due to the numerous randomly initialized parameters to be fully trained. To address this, TF-TAS \cite{tftas} provides a DSS indicator to evaluate the subnet performance without training all supernet parameters. Compared with prior methods, our work highlights the one-stage search for compact architectures in off-the-shelf pre-trained ViTs, thus saving high costs for supernet training and an extra sparsity search.

\paragraph{Vision Transformer Pruning.}
Unlike the pruning for Convolutional Neural Networks (CNNs) \cite{wen2016learning, tang2024enhanced}, the pruning for ViTs contains more prunable components, \eg, Patch Embedding, Patch Token, Multi-Head Self-Attention (MHSA), and MLPs, \etc. S$^2$VITE \cite{s2vite} presents a pruning-and-growing strategy with 50\% ratio to find the sparsity in several dimensions. WDpruning \cite{wdpruning} performs TP via binary masks and injected classifiers, meanwhile designing a learnable pruning threshold based on the parameter constraint. ViT-Slim \cite{vit_slim} employs soft masks with $\ell_1$ penalty by a manually-set global budget threshold for TP. UVC \cite{uvc} jointly combines different techniques to unify VTC. Compared with previous methods, our method features the entanglement of importance and sparsity distributions to jointly determine the prunability of each unit, and the adaptive conversion from the continuous score distribution into a discrete one, thus being able to better balance the sparsity constraint and model performance.

\paragraph{Masked Image Modeling.}
Masked Image Modeling (MIM) \cite{mae, simmim} is a self-supervised learning strategy for augmenting pre-training models, aiming to reconstruct the masked patches in the input image. Several works have explored the representation learning ability of MIM in the pre-training models for downstream tasks\cite{mimdet, mim_distill, m2ae, alpha_darts}, by predicting the patch- or feature-level labels. Differently, our work targets at the compression of pre-trained ViTs, and focuses on utilizing the representation learning ability of MIM to progressively improve the dimension-reduced feature space.
\vspace{-1.5em}
\section{The Proposed Approach}
\vspace{-0.5em}
\label{sec:approach}
We first revisit the two-stage search paradigm and identify its problem, then propose a one-stage counterpart containing a bi-mask weight-sharing scheme with an improved optimization objective, to learn importance and sparsity scores simultaneously. Finally, PMIM is designed to boost the dimension-reduced features and enhance search performance.

\begin{figure*}[t]
\vspace{-3em}
    \centering
    \small
    \resizebox{0.86\linewidth}{!}{
    \includegraphics{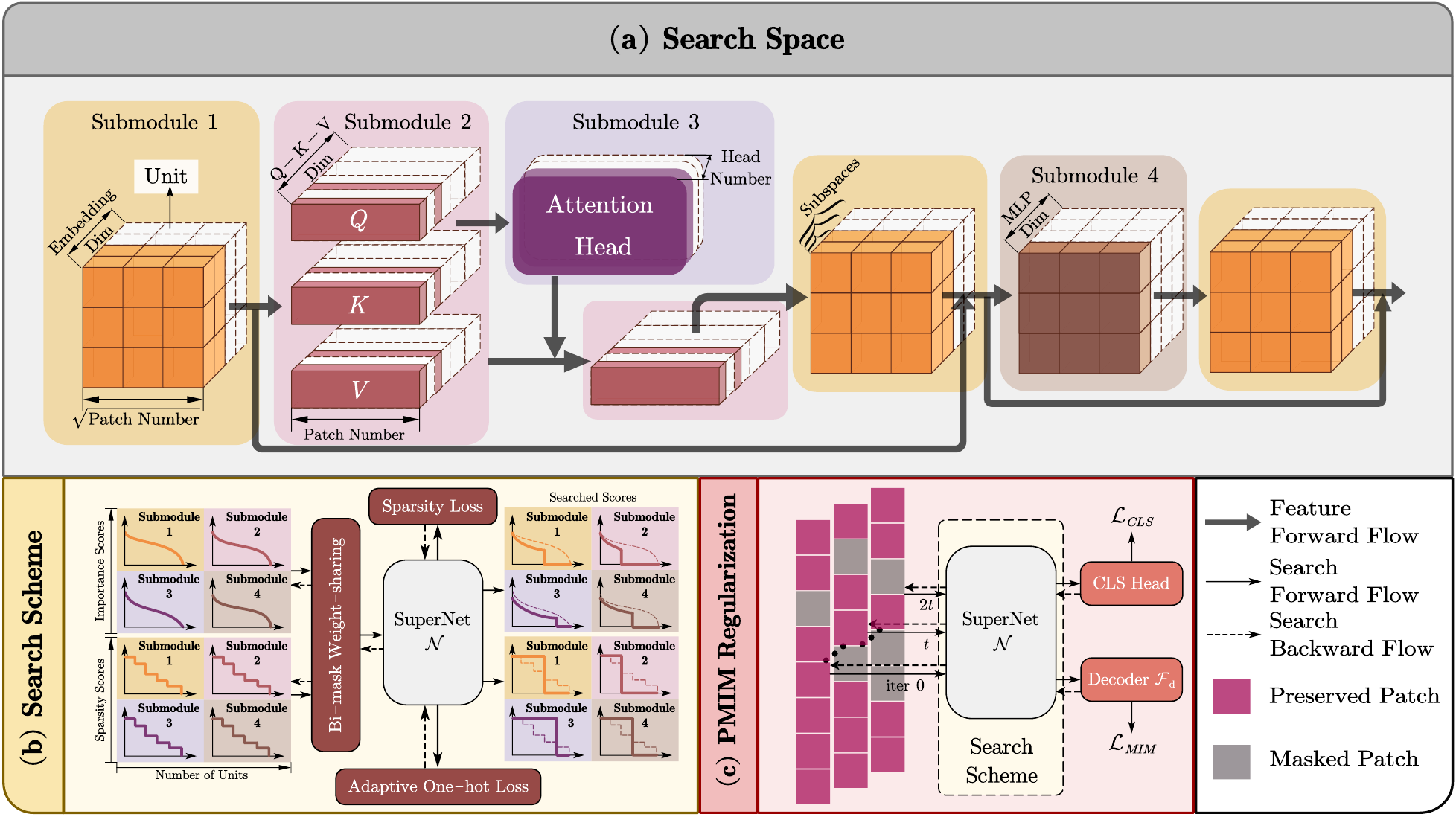}}
    \vspace{-0.5em}
    \caption{The overview of OFB search framework, including the design of search space, search scheme, and regularization scheme. (a) For the search space, we consider four types of submodules. (b) For the search scheme, we simultaneously learn the importance score $\mathcal{S}$ and the sparsity score $\mathcal{V}$ based on the bi-mask weight-sharing strategy, under the guidance of an adaptive one-hot loss. (c) The PMIM technique is developed to augment the pruned feature space, which introduces a progressive masking strategy to MIM for better regularization.}
    \label{fig:overview}
    \vspace{-1.5em}
\end{figure*}

\vspace{-0.5em}
\subsection{Problem Formulation}
\label{sec:pre}
\vspace{-0.5em}

Prior VTC works mainly focus on searching for an optimal sub-network given a supernet and the resource constraint. Let $\mathcal{N}(\mathcal{A}, W)$ denote the supernet, where $\mathcal{A}$ and $W$ refer to the architecture search space and weights of the supernet, respectively. The search for the optimal architecture can be generally formulated as a two-stage problem in Eq. (\ref{nas_obj}).
\begin{equation}\label{nas_obj}
\footnotesize
\begin{aligned}
    &\text{{\bf Stage 1. }} \mathcal{S}_\mathcal{A} = \boldsymbol{f}(W; \mathcal{A}); \\ 
    &\text{{\bf Stage 2. }} \mathop{\min}\limits_{\alpha \in \mathcal{A}, W} \mathcal{L}_{val}(\alpha, W; \mathcal{S}_\mathcal{A}),\,\, \text{s.t.}\,\, g(\alpha) \leq \tau,
\end{aligned}
\end{equation}
where $\boldsymbol{f}$ denotes the criterion to evaluate (\eg, TP) or preset (\eg, TAS) the importance score of each unit $\mathcal{S}_{\mathcal{A}}$ based on the search space, and $\mathcal{L}_{val}$ denotes the loss on the validation dataset. The $g$ and $\tau$ represent the computation cost and the corresponding constraint, respectively. In the first stage, the importance distribution is globally learned from the weights of the supernet (\eg, TP) or naturally encoded in the training mode of the supernet (\eg, TAS). Based on the (piecewise) continuous importance distribution, the architecture parameter $\alpha$ is optimized to satisfy the sparsity constraint via the global threshold or evolutionary algorithms in the second stage, which can be viewed as a discretization process. Since the importance distribution is fixed during the search, the gap between the importance distribution and the searched discrete sparsity distribution (pruned or not pruned) may cause the sub-optimal search result. In other words, the pre-assessed importance score may change with the discretization of other units, and cannot fully represent the actual importance distribution in the searched model. Therefore, a better indicator to assess the unit's prunability could be an adaptively discretized score, that bridges the gap between the importance and sparsity distributions. 

Inspired by DARTS \cite{darts}, which designs a differentiable scheme to relax the pruning choice of each subspace to a $\mathrm{softmax}$-activated probability over all subspaces in one dimension, we further develop a bi-mask scheme to learn the prunability of units in the pre-trained ViTs. In this scheme, the importance and sparsity scores are learned simultaneously in a differentiable manner to jointly determine the unit's prunability. In other words, the search objective is formulated into a one-stage problem, as shown in Eq. (\ref{our_objective}).
\begin{equation}\label{our_objective}
\footnotesize
    \mathop{\min}\limits_{\mathcal{S}, \mathcal{V}, W}\mathcal{L}_{train}(\mathcal{S}, \mathcal{V}, W), \,\, \text{s.t.}\,\, g(\mathcal{V}) \leq \tau,
\vspace{-0.5em}
\end{equation}
where importance scores $\mathcal{S}$, sparsity scores $\mathcal{V}$, and supernet weights $W$ are continually optimized to find the optimal subnet. Consequently, the model is evaluated and searched in a single stage, which is different from the prior works separately evaluating and searching subnets. The optimization of Eq. (\ref{our_objective}) comprises a bi-mask weight-sharing strategy to assign units with prunability scores, and an adaptive one-hot loss to achieve the target sparsity (See Sec. \ref{sec:bi-mask} and \ref{sec:one-hot}).

\begin{table}[t]
    \centering
    \resizebox{\linewidth}{!}{
    \begin{tabular}{lccccc}
        \toprule
         Model & Q-K-V ratio & MLP ratio & Head number & P. E. ratio \\ \hline
         DeiTs \cite{deit} & (1/4, 1, 1/8) & (1/4, 1, 1/8) & (1, num\_heads, 2) & (1/2, 1, 1/32)\\
         Swin-Ti \cite{swin} & (1/4, 1, 1/8) & (1/4, 1, 1/8) & (1, num\_heads, 2) & (1/2, 1, 1/32)\\
         \bottomrule
    \end{tabular}}
    \vspace{-0.5em}
    \caption{Search spaces of DeiTs \cite{deit} and Swin-Ti \cite{swin}. Tuples in parentheses denote the lowest value, the highest value, and step size. \textbf{Note}: the step size of P. E. is smaller for its more significant impact on multiple layers (See Fig. \ref{fig:overview}) and compression performance.}
    \label{tab:search_space}\vspace{-2em}
\end{table}

\paragraph{Search Space.} We first follow the TAS paradigm to construct a discrete search space for all units in each prunable submodule, including Q-K-V ratio, MLP ratio, Head number, and Patch Embedding (P. E.), as described in Table \ref{tab:search_space}. Then, the search space is relaxed in a differentiable manner.

\vspace{-1.0mm}
\subsection{Bi-mask Weight-sharing Strategy}
\label{sec:bi-mask}
\vspace{-1.0mm}
In order to assess the prunability of each unit, we introduce a bi-mask weight-sharing strategy in the search process. Each prunability score is represented by the value of the designed bi-mask $m_{ij}$ that considers both the importance score and the sparsity score, which can be illustrated as follows:
\begin{equation}
\footnotesize
\label{prunability_score}
    m_{ij}\left( t \right) =\lambda \left( t \right) \mathcal{S} _{ij} +\left[ 1-\lambda \left( t \right) \right] \mathcal{V} _{ij}\left( \alpha \right), 
\end{equation}
where the subscript index $ij$ denotes the $j$-th unit in the $i$-th prunable submodule. $\lambda(t)$ denotes the time-varying weight coefficient of the importance score. Specifically, $\lambda$ linearly changes from one to zero until the model finishes searching. The motivation behind this is two-fold. From the lens of score optimization, \ie, the backward process, the sparsity score of each unit is related to its importance rank among all units in the same submodule. Therefore, before transmitting a large gradient to the sparsity score, more attention should be paid to learning a reliable importance score. Since the model weights $W$ are well trained in the supernet, the importance score could be learned appropriately in several epochs, thus providing a relatively accurate importance rank for the assignment of the sparsity score to each unit. After obtaining a relatively accurate importance score, optimization should be focused more on the sparsity score to make the pruning decision. From the lens of the discretization process, \ie, the forward process, the search target is to learn a discrete score for each unit; thus, the searched score should approach an approximately binary (0/1) distribution, which is exactly the desired distribution of the sparsity score $\mathcal{V}$. Therefore, the learning focus of the prunablility score should be gradually transferred from importance to sparsity during searching.

As for the importance score $\mathcal{S}$, inspired by ViT-Slim \cite{vit_slim}, we introduce a soft mask that is randomly initialized and learnable in each unit, to indicate its contribution to supernet performance. The importance score is normalized to vary between $(0, 1)$ via $\mathrm{sigmoid}$. As for the sparsity score $\mathcal{V}$, we leverage the architecture parameter $\alpha$ to generate the sparsity score of each unit. Given $\alpha$, $\mathcal{V}$ is computed via $\mathrm{softmax}$ to indicate the preserving potential, as formulated in Eq. (\ref{sparsity_score}):
\begin{equation}
\vspace{-0.6em}
\footnotesize
    \mathcal{V} _{ij}(\alpha )=\frac{\sum\nolimits_{k=\lfloor j/\varDelta _i \rfloor}^{\|\alpha_i\|_0}{\exp \left( \alpha _{ik} \right)}}{\sum\nolimits_{k=0}^{\left\| \alpha _{i,:} \right\| \!\:\!_0}{\exp \left( \alpha _{ik} \right)}}=\sum\nolimits_{k=\lfloor j/\varDelta _i \rfloor}^{\|\alpha_i\|_0}{p_{ik}},
    \label{sparsity_score}
\end{equation}
where $\alpha_i$ is the architecture parameter vector of the $i$-th submodule to parameterize the sub-space into a continuous space. $p_{ik}$ and $\varDelta _i$ represent the step-wise normalized architecture score and the step size in the search space of the $i$-th submodule, respectively, where $p_{ik}=\mathrm{softmax}_k \left( \alpha _{ik} \right)$. Note that the weights in all sub-spaces of the submodule are shared as DARTS \cite{darts} does; therefore, the sparsity score of each unit is the sum of those shared architecture scores, making the sparsity distribution piecewise continuous. Unlike previous differentiable search methods uniformly initializing $\alpha$ for the randomly initialized supernet, our method randomly initializes $\alpha$ to reduce the inductive bias.

As for the weight-sharing strategy in differentiable search, considering the units with higher importance scores are more likely to be preserved, sparsity scores of more important units should be correspondingly higher than those less important (a high sparsity score means high preserving potential). Thus, at forward steps, the units in each submodule are reorganized \wrt their importance score rank and assigned sparsity scores in a descending order, as shown in Fig. \ref{fig:difference}\textcolor{red}{d}.

\subsection{Adaptive One-hot Loss}
\label{sec:one-hot}
Given bi-masks as introduced above, which soft mask units to indicate the prunability during searching, the optimization target of these masks is another important issue. In Sec \ref{sec:one-hot}, we present an adaptive one-hot loss to solve this problem.

Considering $m$ is derived from $\mathcal{S}$ and $\mathcal{V}$, the optimization could be decomposed into two parts. As for the importance score, the aim of $\mathcal{S}$ is to learn an importance rank according to the unit's impact on model performance under the sparsity constraint. Thus, we follow ViT-Slim \cite{vit_slim} to regularize $\mathcal{S}$ with $\ell_1$ norm to drive unimportant units towards low-ranking and zero-score distribution, \ie, $\mathcal{L}_\mathcal{S}=\|\mathcal{S}\| \!\!\:_1$.

As for the sparsity score, the aim of $\mathcal{V}$ is to learn a binary distribution as the unit pruning choice. In other words, the target label of each $p_i$ is ideally designed as a progressively shrunk one-hot vector, with no prior information about the one-hot index, thus being difficult to impose a definite target to $\mathcal{V}$ and $\alpha$. To address this, we propose to regularize the sparsity score by introducing an alternative constraint that aligns the entropy and variance of $p_i$ with one-hot vectors. The motivation stems from the invariance of the two properties in one-hot vectors, regardless of the one-hot index. Especially, the entropy of any one-hot vector always equals zero, while the variance solely depends on the dimension number. The regularization \wrt $p$ is formulated as follows:
\begin{equation}
\footnotesize
\label{regularization}
\begin{aligned}
R\left( p \right) &=\sum_{i=1}^M{\left[ \mathcal{H} \left( p _i \right) +\varPsi \left( p _i \right) \right]} \\ &=\sum_{i=1}^M{\left[ -p_{i}^{T}\log \left( p_i \right) +\tan \left( \frac{\pi}{2}-\pi \omega _i \right) \right]},
\end{aligned}
\end{equation}
where $M$ denotes the number of searchable submodules in {\small $\mathcal{N}(\mathcal{A}, W)$, $\omega _i=\sigma _{i}/\sigma _{i}^{t}$}
with $\sigma _{i}$, $\sigma _{i}^{t}$, and $\omega_i$ meaning the measured, target and normalized variances of $p_i$, respectively, where {\small $\sigma _{i}^{t}=\left( \left\| \alpha _i \right\| _{\!\:\!0}-1 \right) /\left\| \alpha _i \right\| _{\!\:\!0}^{2}$}.
(See Appendix \textcolor{red}{A} for more detailed explanations). In addition to Eq. (\ref{regularization}), $\mathcal{V}$ is also constrained by the sparsity constraint, $\tau$. Therefore, the total regularization objective of $\mathcal{V}$ is formulated as follows:
\begin{equation}
\footnotesize
    \label{sparsity_loss}
    \mathcal{L} _{\mathcal{V}}=\mu_1 R\left( p \right) +\mu_2 \left\| g\left( \mathcal{V} \right) -\tau \right\| \!\!\:_2,
\end{equation}
where $\mu_1$ and $\mu_2$ are the weight coefficients to balance two items. Note that during search, the ground-truth value $\sigma^t$ would change with the decrease of $\|\alpha\|\!\!\:_0$ when the pruning happens in $\alpha$. Thus, $\mathcal{L} _{\mathcal{V}}$ is adaptive to the pruning process in the search stage. The pruning process in one dimension (\eg, the $i$-th submodule) is triggered by the condition that {\small $\left( p_i \right) _{\min}\leqslant \eta \cdot \bar{p}_i$}, where $\eta$ is the scaling factor and $\bar{p}_i$ is the mean of $p_i$, \ie, {\small $\bar{p}_i=1/\left\| \alpha _i \right\| \!\!\:_0$}.
By Eq. (\ref{sparsity_loss}) and the proposed pruning strategy, the units with the lowest mask values can be progressively removed, thus accelerating search process.

Based on the above analysis, the regularization items for the bi-mask can be summarized as, {\small $\mathcal{L} _m(\mathcal{V} ,\mathcal{S} )=\mathcal{L} _{\mathcal{V}}+\mu _3\mathcal{L} _{\mathcal{S}}=\mu _1R\left( p \right) +\mu _2\left\| g\left( \mathcal{V} \right) -\tau \right\| \!\!\:_2+\mu _3\| \mathcal{S} \| \!\!\:_1$},
where $\mu_3$ denotes the weight coefficient of $\mathcal{L}_\mathcal{S}$. Consequently, the objective in Eq. (\ref{our_objective}) is transformed into the following equation:
\begin{minipage}{\linewidth}
\vspace{-1em}
\begin{algorithm}[H]
\small
\caption{Once for Both (OFB).}
\label{algo}
\renewcommand{\algorithmicensure}{\textbf{Input:}}
\begin{algorithmic}[1]

\Ensure{Pre-trained ViT $\mathcal{N}$, Decoder $\mathcal{F}_d$, Search Space $\mathcal{A}$, Dataset $\mathcal{D}$, Masking Ratio $\gamma$, Pruning Interval $\Delta\mathrm{T}$, Target Pruning Ratio $\tau$;}
\State{Initialize Importance Score Set $\{\mathcal{S}\}$ and Architecture Parameter Set $\{\alpha\}$ according to $\mathcal{A}$;}
\State{Compute $m$ via Eq. (\ref{prunability_score}) and insert $m$ to units in the search space as soft masks;}
\For {each training iteration $t$}
\State \begin{varwidth}[t]{\linewidth}
Sample $b_t\sim\mathcal{D}$ and random mask $\gamma$ patches;
\end{varwidth}
\State \begin{varwidth}[t]{0.95\linewidth}
Forward [$\mathcal{N}(W; m)$; $\mathcal{F}_d$] with masked $b_t$;
\end{varwidth}
\State \begin{varwidth}[t]{0.95\linewidth}
Update $\mathcal{S}$, $\mathcal{V}, W$ by optimizing Eq. (\ref{mim_objective});
\end{varwidth}
\State \begin{varwidth}[t]{\linewidth}
Linearly update $\gamma$ and $\lambda$;
\end{varwidth}
\State \begin{varwidth}[t]{\linewidth}
Update $m$ via Eq. (\ref{prunability_score});
\end{varwidth}
\If{not $finish\,search$ and ($t$ mod $\Delta\mathrm{T}== 0$)} 
\For {each submodule $\alpha_i$ in $\mathcal{A}$}
\If{$\left( p_i \right) _{\min}\leqslant \eta \cdot \bar{p}_i$}
\State{Prune the units whose $p_i \leqslant \eta \cdot \bar{p}_i$;}
\EndIf
\EndFor
\EndIf
\EndFor\\
\Return{the pruned ViT satisfying the target sparsity.}
\end{algorithmic}
\end{algorithm}
\vspace{-1em}
\end{minipage}
\begin{equation}
\footnotesize
\label{transformed_objective}
    \min_{\mathcal{S}, \mathcal{V} ,W} \mathcal{L} _{train}(\mathcal{S}, \mathcal{V}, W)+\mathcal{L} _m(\mathcal{V} ,\mathcal{S}).
\end{equation}

\begin{table*}[t]
\vspace{-2em}
\centering
\resizebox{\textwidth}{!}{
\begin{tabular}{y{60}x{60}y{100}x{80}x{80}x{80}x{80}x{80}x{80}}
\toprule
 & Method & Model & \#Param (M) & FLOPs (B) & Top-1 (\%) & Top-5 (\%) & GPU Days \\ \midrule
\multirow{22}{*}{DeiT-S} & Baseline & DeiT-S \cite{deit} & 22.1 & 4.6 & 79.8 & 95.0 & - \\ \cmidrule(){2-8} 
 & \multirow{4}{*}{TP} & SSP-T$^\dagger$ \cite{s2vite} & 4.2 & 0.9 & 68.6 & - & - \\
 &  & S$^2$ViTE-T$^\dagger$ \cite{s2vite} & 4.2 & 0.9 & 70.1 & - & - \\
 &  & WDPruning-0.3-12$^\dagger$ \cite{wdpruning} & 3.8 & 0.9 & 71.1 & 90.1 & - \\
 &  & S$^2$ViTE-S \cite{s2vite} & - & 2.1 & 74.8 & - & - \\ \cmidrule(){2-8} 
 & \multirow{1}{*}{TAS} & ViTAS-B \cite{vitas} & - & 1.0 ($\downarrow$78\%) & 72.4 ($\downarrow$7.4)& 90.6 ($\downarrow$4.4)& 32 \\ \cmidrule(){2-8} 
 &  \cellcolor[gray]{0.9}& \cellcolor[gray]{0.9} OFB & \cellcolor[gray]{0.9}4.4 ($\downarrow$80\%) &\cellcolor[gray]{0.9} 0.9 ($\downarrow$80\%)&\cellcolor[gray]{0.9} 75.0 ($\downarrow$4.8) &\cellcolor[gray]{0.9} 92.3 ($\downarrow$2.7)&\cellcolor[gray]{0.9} 1 \\ \cmidrule(){2-8} 
 & \multirow{3}{*}{TAS} & AutoFormer-Ti \cite{autoformer} & 5.7 ($\downarrow$74\%)& 1.3 ($\downarrow$72\%)& 74.7 ($\downarrow$5.1)& 92.6 ($\downarrow$2.4)& 24 \\
 &  & ViTAS-C \cite{vitas} & 5.6 ($\downarrow$75\%)& 1.3 ($\downarrow$72\%)& 74.7 ($\downarrow$5.1)& 91.6 ($\downarrow$3.4)& 32 \\
 &  & TF-TAS-Ti \cite{tftas} & 5.9 ($\downarrow$73\%)& 1.4 ($\downarrow$70\%)& 75.3 ($\downarrow$4.5)& 92.8 ($\downarrow$2.2)& 0.5 \\ \cmidrule(){2-8} 
 & \multirow{2}{*}{Lightweight} & DeiT-Ti \cite{deit} & 5.7 & 1.3 & 72.2 & 91.1 & - \\
 &  & TNT-Ti \cite{tnt} & 6.1 & 1.4 & 73.9 & 91.9 & - \\ \cmidrule(){2-8} 
 & \cellcolor[gray]{0.9} & \cellcolor[gray]{0.9}OFB &\cellcolor[gray]{0.9} 5.3 ($\downarrow$76\%)&\cellcolor[gray]{0.9} 1.1 ($\downarrow$76\%)&\cellcolor[gray]{0.9} 76.1 ($\downarrow$3.7)&\cellcolor[gray]{0.9} 92.8 ($\downarrow$2.2)&\cellcolor[gray]{0.9} 1 \\ \cmidrule(){2-8} 
 & \multirow{5}{*}{TP} & SSP-S \cite{s2vite} & 14.6 ($\downarrow$34\%)& 3.1 ($\downarrow$33\%)& 77.7 ($\downarrow$2.1)& - & - \\
 &  & S$^2$ViTE-T \cite{s2vite} & 14.6 ($\downarrow$34\%)& 2.7 ($\downarrow$41\%)& 78.2 ($\downarrow$1.6)& - & - \\
 &  & ViT-Slim \cite{vit_slim} & 11.4 ($\downarrow$48\%)& 2.3 ($\downarrow$50\%)& 77.9 ($\downarrow$1.9)& 94.1 ($\downarrow$0.9)& 1.8 \\
 &  & WDPruning-0.3-12 \cite{wdpruning} & - & 2.6 ($\downarrow$43\%)& 78.4 ($\downarrow$1.4)& - & - \\ \cmidrule(){2-8} 
 & \multirow{2}{*}{Lightweight} & HVT \cite{hvt} & - & 2.4 & 78.0 & - & - \\
 & & PVT-Ti \cite{pvt} & 13.2 & 1.9 & 75.1 & - & - \\ \cmidrule(){2-8} 
 & \cellcolor[gray]{0.9} & \cellcolor[gray]{0.9}OFB & \cellcolor[gray]{0.9}8.0 ($\downarrow$64\%)& \cellcolor[gray]{0.9}1.7 ($\downarrow$63\%)&\cellcolor[gray]{0.9} 78.0 ($\downarrow$1.8)&\cellcolor[gray]{0.9} 93.9 ($\downarrow$1.1)&\cellcolor[gray]{0.9} 1 \\ \midrule
\multirow{19}{*}{DeiT-B} & Baseline & DeiT-B \cite{deit} & 86.6 & 17.5 & 81.8 & 95.6 & - \\ \cmidrule(){2-8} 
 & \multirow{4}{*}{TAS} & AutoFormer-S \cite{autoformer} & 22.9 ($\downarrow$74\%)& 5.1 ($\downarrow$71\%)& 81.7 ($\downarrow$0.1)& 95.7 ($\uparrow$0.1)& 24 \\
 &  & ViTAS-F \cite{vitas} & 27.6 ($\downarrow$68\%)& 6.0 ($\downarrow$66\%)& 80.5 ($\downarrow$1.3)& 95.1 ($\downarrow$0.5)& 32 \\
 &  & GLiT-S \cite{glit} & 24.6 ($\downarrow$72\%)& 4.4 ($\downarrow$75\%)& 80.5 ($\downarrow$1.3)& - & - \\ 
 &  & TF-TAS-S \cite{tftas} & 22.8 ($\downarrow$74\%)& 5.0 ($\downarrow$71\%)& 81.9 ($\uparrow$0.1)& 95.8 ($\uparrow$0.2)& 0.5 \\ \cmidrule(){2-8} 
 & \multirow{2}{*}{TP} & DynamicViT-S$^\star$ \cite{dynamicvit} & 22.0 & 4.0 & 79.8 & - & - \\
 &  & ViT-Slim$^\star$ \cite{vit_slim} & 17.7 & 3.7 & 80.6 & 95.3 & 3 \\ \cmidrule(){2-8} 
 & Lightweight & DeiT-S \cite{deit} & 22.1 & 4.6 & 79.8 & 95.0 & - \\\cmidrule(){2-8} 
 & \cellcolor[gray]{0.9} & \cellcolor[gray]{0.9}OFB & \cellcolor[gray]{0.9}17.6 ($\downarrow$80\%)& \cellcolor[gray]{0.9}3.6 ($\downarrow$79\%)& \cellcolor[gray]{0.9}80.3 ($\downarrow$1.5)&\cellcolor[gray]{0.9} 95.1 ($\downarrow$0.5)&\cellcolor[gray]{0.9} 2.9 \\ \cmidrule(){2-8} 
 & \multirow{3}{*}{TAS} & AutoFormer-B \cite{autoformer} & 54.0 ($\downarrow$38\%)& 11.0 ($\downarrow$37\%)& 82.4 ($\uparrow$0.6)& 95.7 ($\uparrow$0.1)& 24 \\
 &  & GLiT-B \cite{glit} & 96.0 ($\uparrow$11\%)& 17.0 ($\downarrow$3\%)& 82.3 ($\uparrow$0.5)& - & - \\
 &  & TF-TAS-S \cite{tftas} & 54.0 ($\downarrow$38\%)& 12.0 ($\downarrow$31\%)& 82.2 ($\uparrow$0.4)& 95.6 ($\downarrow$0.0)& 0.5 \\ \cmidrule(){2-8} 
 & \multirow{4}{*}{TP} & VTP \cite{vtp} & - & 10.0 ($\downarrow$43\%)& 80.7 ($\downarrow$1.1)& 95.0 ($\downarrow$0.6)& - \\
 &  & S$^2$ViTE-B \cite{s2vite} & 56.8 ($\downarrow$34\%)& 11.7 ($\downarrow$33\%)& 82.2 ($\uparrow$0.4)& - & - \\
 &  & ViT-Slim \cite{vit_slim} & 52.6 ($\downarrow$39\%)& 10.6 ($\downarrow$39\%)& 82.4 ($\uparrow$0.6)& 96.1 ($\uparrow$0.5)& 3 \\
 &  & WDPruning-0.3-11 \cite{wdpruning} & - & 9.9 ($\downarrow$43\%)& 80.8 ($\downarrow$1.0)& 95.4 ($\downarrow$0.1) & - \\ \cmidrule(){2-8} 
 & \cellcolor[gray]{0.9} & \cellcolor[gray]{0.9}OFB & \cellcolor[gray]{0.9}43.9 ($\downarrow$49\%)& \cellcolor[gray]{0.9}8.7 ($\downarrow$50\%) &\cellcolor[gray]{0.9} 81.7 ($\downarrow$0.1)&\cellcolor[gray]{0.9} 95.8 ($\uparrow$0.2)&\cellcolor[gray]{0.9} 4 \\ \bottomrule
\end{tabular}}
\vspace{-1em}
\caption{Compression results of DeiT models on ImageNet-1K. $\dagger$ and $\star$ indicate that the compression model is based on DeiT-Ti and DeiT-S \cite{deit}, respectively. $\uparrow/\downarrow$ refers to the increase$/$decrease ratio.}
\label{deit_imagenet}
\vspace{-1.5em}
\end{table*}

\subsection{Progressive MIM}
As discussed above, during the search, pruning units can reduce training costs but may impair model performance. Consequently, the evaluation of the importance and sparsity scores for the remaining units may be unreliable, especially in aggressive compression scenarios where a large proportion of units are pruned. To address this, we propose to further enhance the representative capability of the remaining features using the MIM technique \cite{simmim} during the search. Yet, deploying MIM in a well-trained ViT requires considering the characteristics of VTC. Specifically, using a large masking ratio similar to SimMIM \cite{mae, simmim} may be inappropriate for VTC as the model would deviate from the original classification-oriented feature space to the reconstruction-oriented space, losing the original contextual information. Motivated by this, we propose a progressive MIM strategy for searching, where the masking ratio gradually increases until reaching the preset threshold. Consequently, as the pruning ratio gets larger, the feature regularization is strengthened with negligible additional cost \cite{simmim}, thereby maintaining the representation ability and a reliable prunability evaluation of the remaining units. Then, the optimization objective is updated to:
\begin{equation}
    \footnotesize
    \label{mim_objective}
    \min_{\mathcal{S}, \mathcal{V} ,W} \mathcal{L} _{train}(\mathcal{S}, \mathcal{V} ,W)+\mathcal{L} _m(\mathcal{V} ,\mathcal{S}) + \mathcal{L}_{rec}(\mathcal{S}, \mathcal{V}, W; \gamma),
\end{equation}
where $\mathcal{L}_{rec}$, $\gamma$ denote the reconstruction loss and the masking ratio, respectively. Alg. (\ref{algo}) shows the overall algorithm.

\section{Experiments}
\vspace{-0.5em}
\label{sec:experiment}

We first present the compression results for DeiTs \cite{deit} and Swin Transformer \cite{swin} on ImageNet-1K \cite{imagenet}. Next, we analyze the generalization ability of the searched models on downstream benchmarks. Lastly, we explore the contribution of each individual component to search performance, followed by visualizations of the search process and results. The implementation details are described in Appendix \textcolor{red}{C}.

\begin{figure}[t]
    \centering
    \small
    \resizebox{0.66\linewidth}{!}{\includegraphics{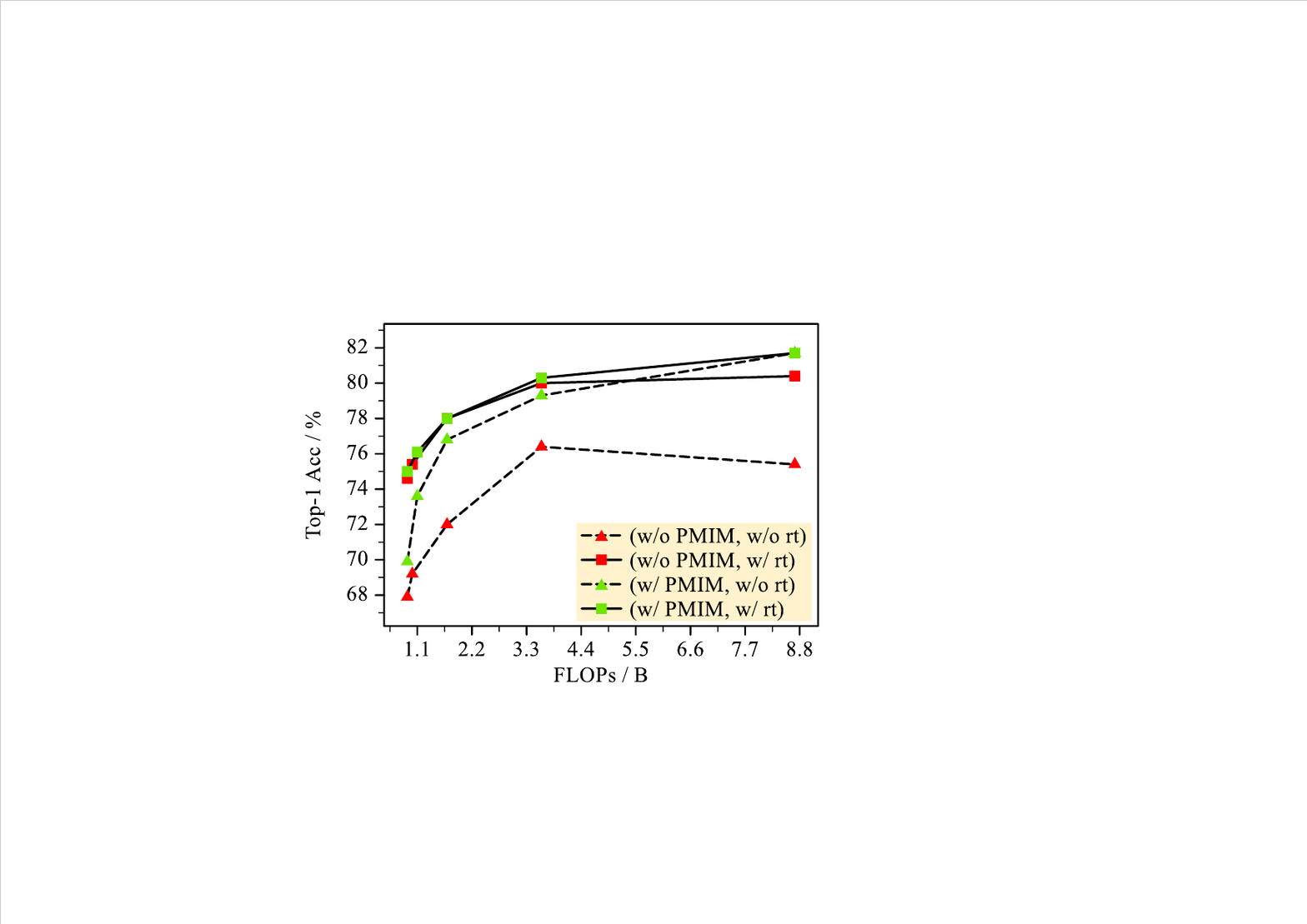}}
    \vspace{-1em}
    \caption{Performance of the searched DeiT models with/without retraining by employing/not employing PMIM during searching.}
    \vspace{-3em}
    \label{fig:PMIM}
\end{figure}

\vspace{-0.5em}
\subsection{Results on ImageNet}
\vspace{-0.5em}
We summarize the main results of the DeiT family in Table \ref{deit_imagenet} and compare them with various existing TAS and TP methods. For a fair comparison, we refer to the results of other works from official papers. It can be observed that, in general, the compressed models based on OFB achieve higher accuracy and larger compression ratios. On low-accuracy regimes, OFB achieves the highest accuracy and compression ratio for DeiT-S among all compression methods and lightweight models, obtaining at most 2.6\% accuracy increase than TAS-based methods with $\sim$80\% FLOPs reduction. On high-accuracy regimes, the compressed DeiT-B based on OFB also outperforms both TAS- and TP-based methods. In particular, OFB achieves 0.9\% higher Top-1 than WDPruning \cite{wdpruning}, and controls the accuracy decrease at 0.1\% Top-1 with 50\% reductions in parameters and FLOPs. The other compression results also validate the consistent effectiveness and versatility of our method in different DeiT models. As for search efficiency, we report the search time cost of several TAS- and TP-based methods in the last column of Table \ref{deit_imagenet} for comparisons. Compared with AutoFormer \cite{autoformer} and ViTAS \cite{vitas}, OFB reduces search time greatly at various budgets, meanwhile achieving comparable or better performance. Compared with the efficient compression methods, \eg, TF-TAS \cite{tftas} and ViT-Slim \cite{vit_slim}, OFB can achieve comparable within similar search time. One thing to note is that for the compressed DeiT-B-8.7B, after searching, it already reaches the same high performance as the retrained one. As visualized in Fig. \ref{fig:PMIM}, we compare the model performance with and without retraining process by employing or not employing PMIM. It is observed that the performance gap of the models with PMIM before and after retraining is significantly smaller than that of the models without PMIM. Therefore, when considering the cost induced by retraining, OFB is more efficient than TF-TAS and ViT-Slim, saving the retraining cost meanwhile maintaining high performance.

\begin{table}[t]
\vspace{-2em}
\small
\centering
\tablestyle{2pt}{1.}
\resizebox{\linewidth}{!}{
\begin{tabular}{lcccccc}
\toprule
 \multirow{2}{*}{Model} & \#Param & FLOPs & Top-1 & Top-5 & GPU \\ 
 & (M)  & (B) & (\%) &  (\%) & Days \\\midrule
 Swin-Ti \cite{swin} & 28.3 & 4.5 & 81.3 & 95.5 & -\\
 \rowcolor[gray]{0.9}OFB & 6.1 ($\downarrow$78\%) & 1.0 ($\downarrow$78\%)& 76.5 ($\downarrow$4.8)& 93.1 ($\downarrow$2.4) & 1.1\\
 \rowcolor[gray]{0.9}OFB & 16.4 ($\downarrow$42\%) & 2.6 ($\downarrow$42\%)& 79.9 ($\downarrow$1.4)& 94.6 ($\downarrow$0.9) & 1.3\\\midrule
  ViT-Slim \cite{vit_slim} & 19.4 ($\downarrow$31\%)& 3.4 ($\downarrow$24\%)& 80.7 ($\downarrow$0.6)& 95.4 ($\downarrow$0.1) & -\\
 \rowcolor[gray]{0.9}OFB & 18.9 ($\downarrow$33\%) & 3.1 ($\downarrow$31\%)& 80.5 ($\downarrow$0.8)& 94.8 ($\downarrow$0.7) & 1.4\\ \bottomrule
\end{tabular}}
\vspace{-1em}
\caption{Compression results of Swin-Ti \cite{swin} on ImageNet-1K.}
\label{swin_imagenet}
\vspace{-1em}
\end{table}

\begin{table}[t]
\small
\centering
\tablestyle{2pt}{1.}
\resizebox{\linewidth}{!}{
\begin{tabular}{lccc|lcccc}
\multirow{2}{*}{Model} & F. & C10 & C100 & \multirow{2}{*}{Model} & F. & mIoU & mAcc & aAcc \\ 
 & (B) & (\%) & (\%) & & (B) & (\%) & (\%) &  (\%) \\\shline
DeiT-S \cite{deit} & 4.6 & 98.6 & 87.8 & SETR-S \cite{setr} & 4.6 & 73.0 & 81.4 & 95.1 \\
\rowcolor[gray]{0.9}OFB & 1.7 & 98.7 & 88.4 & OFB & 3.6 & 73.9 & 83.1 & 95.2 \\
\end{tabular}}
\vspace{-1em}
\caption{Performance of DeiT-S \cite{deit} and compressed models on CIFAR-10 (C10), CIFAR-100 (C100) \cite{cifar} and Cityscape \cite{cityscape}.}
\label{deit_cifar}
\vspace{-2em}
\end{table}

As for the compression of Swin Transformer \cite{swin}, we choose Swin-Ti as the baseline model to validate the effectiveness of OFB. The results are listed in Table \ref{swin_imagenet}. From the table, it can be observed that the compression performance of Swin-Ti is similar to that of DeiT-S. For example, when the model is compressed with nearly 80\% reductions in FLOPs and parameters (Line 7 in Table \textcolor{red}{2} and Line 2 in Table \ref{swin_imagenet}), they both drop 4.8\% in Top-1 accuracy, while Swin-Ti drops 0.3\% less than DeiT-S in Top-5 accuracy, which further validates the effectiveness of OFB in different ViT structures. 

\begin{table}[t]
\vspace{-2em}
\centering
\subfloat[
\textbf{Weight-sharing Strategy}.
\label{tab:weight_sharing}
]{
\begin{minipage}{0.47\linewidth}{\begin{center}
\tablestyle{1pt}{1.0}
\begin{tabular}{y{28}x{18}x{18}x{18}x{18}}
Case & T1 & T5 & F. & P. \\
\shline
Ordinal & 75.8 & 92.6 & 1.2 & 5.7 \\
Bi-mask & \baseline{\textbf{76.1}} & \baseline{\textbf{92.8}} &\baseline{\textbf{1.1}} & \baseline{\textbf{5.3}}\\
& & &
\end{tabular}
\end{center}}\end{minipage}
}
\hspace{.5em}
\subfloat[
\textbf{Masking Stategy}.
\label{tab:masking_strategy}
]{
\centering
\begin{minipage}{0.45\linewidth}{\begin{center}
\tablestyle{1pt}{1.0}
\begin{tabular}{y{26}x{18}x{18}x{18}x{18}}
Case & T1 & T5 & F. & P.\\
\shline
None & 75.4 & 92.5 & \textbf{1.0} & \textbf{5.1}\\
Cons. & 75.8 & 92.8 & 1.1 & 5.2\\
PMIM & \baseline{\textbf{76.1}} & \baseline{\textbf{92.8}} & \baseline{1.1} &\baseline{5.3}\\
\end{tabular}
\end{center}}\end{minipage}
}
\hspace{1em}
\\
\centering
\subfloat[\textbf{Combinatorial Contributions}.\label{tab:component}]{
\begin{minipage}{\linewidth}
    \begin{center}
        \tablestyle{1pt}{1.}
        \resizebox{0.88\linewidth}{!}{
            \begin{tabular}{y{30}|x{20}x{20}|x{20}x{20}|x{25}x{25}x{25}x{25}}
            \multirow{2}{*}{Case}             & \multicolumn{2}{c|}{Bi-mask} & \multicolumn{2}{c|}{PMIM} & \multirow{2}{*}{\begin{tabular}{c}Top-1\\ (\%)\end{tabular}}       & \multirow{2}{*}{\begin{tabular}{c}Top-5\\ (\%)\end{tabular}}      & \multirow{2}{*}{\begin{tabular}{c}FLOPs\\ (B)\end{tabular}}       & \multirow{2}{*}{\begin{tabular}{c}\#Param\\ (M)\end{tabular}}    \\ 
                           & w/o          & w/          & w/o          & w/           &      &        &        &       \\ \shline
            Baseline       & \checkmark         &                      & \checkmark      &             & 67.5        & 88.2        & 1.18         & 5.7        \\
            \multirow{2}{*}{Bi-mask}        &                      & \multirow{2}{*}{\checkmark}         & \multirow{2}{*}{\checkmark}      &                  & 69.2  & 89.2 & 1.04 & 5.1 \\
                           & & & & & (+1.7) & (+1.0) & (-0.14) & (-0.6)\\
            \rowcolor[gray]{0.9} Bi-mask &         &        &                   &   & 72.8  & 91.4  & 1.09  & 5.3  \\ 
            \baseline{+ PMIM} & \baseline{} & \baseline{\multirow{-2}{*}{\checkmark}} & \baseline{}& \baseline{\multirow{-2}{*}{\checkmark}} &\baseline{(+3.6)} & \baseline{(+2.2)} & \baseline{(+0.05)} & \baseline{(+0.2)} \\
            \end{tabular}}
    \end{center}
\end{minipage}
}
\\
\centering
\subfloat[
\textbf{Regularization Strategy}.
\label{tab:regularization}
]{
\begin{minipage}{\linewidth}{\begin{center}
\tablestyle{1pt}{1.}
\resizebox{0.88\linewidth}{!}{
\begin{tabular}{x{30}x{30}x{30}|x{30}x{30}x{30}x{30}}
\multicolumn{3}{c|}{Regularization Strategy} & \multirow{2}{*}{\begin{tabular}{c@{}}Top-1\\ (\%)\end{tabular}}       & \multirow{2}{*}{\begin{tabular}{c@{}}Top-5\\ (\%)\end{tabular}}      & \multirow{2}{*}{\begin{tabular}{c@{}}FLOPs\\ (B)\end{tabular}}       & \multirow{2}{*}{\begin{tabular}{c@{}}\#Param\\ (M)\end{tabular}}    \\
$\mathcal{H}(\alpha)$ & $\varPsi(\alpha)$ & $\mathcal{L}_\mathcal{S}$ &  &  &  &  \\ \shline
\checkmark &  &  & 76.1 & 92.9 & 1.2 & 5.6 \\
 & \checkmark &  & 74.1 & 91.7 & 0.8 & 4.0 \\
 &  & \checkmark & 74.8 & 91.9 & 1.0 & 4.9 \\
 & \checkmark & \checkmark & 74.3 & 91.9 & \textbf{0.8} & \textbf{4.0} \\
\checkmark &  & \checkmark & \textbf{76.3} & \textbf{93.0} & 1.2 & 5.6 \\
\checkmark & \checkmark &  & 75.6 & 92.6 & 1.1 & 5.2 \\
\rowcolor[gray]{0.9}\checkmark & \checkmark & \checkmark & 76.1 & 92.8 & 1.1 & 5.3\\
\end{tabular}}
\end{center}}\end{minipage}
}
\vspace{-1em}
\caption{\textbf{Ablation studies} with DeiT-S on ImageNet. We report retraining accuracy with target sparsity as 1BFLOPs except for (c), where accuracy w/o retraining is reported to analyze the component contribution in search process. Default settings are marked in gray.}
\label{tab:ablations}
\vspace{-0.5em}
\end{table}

\begin{table}[t]
\vspace{-0.5em}
\centering
\small
\resizebox{1.\linewidth}{!}{
\tablestyle{1pt}{1.05}
\setlength\tabcolsep{2.5pt}
\begin{tabular}{lcc|lcc}
DeiT-S        & \begin{tabular}[c]{@{}c@{}}Throughput\\ (img/s)\end{tabular} & \begin{tabular}[c]{@{}c@{}}Latency\\ (ms)\end{tabular} & DeiT-B       & \begin{tabular}[c]{@{}c@{}}Throughput\\ (img/s)\end{tabular} & \begin{tabular}[c]{@{}c@{}}Latency\\ (ms)\end{tabular} \\ \shline
ViTAS-B       & 2637   & 74    & Auto-S & 735   & 503 \\
\baseline{OFB-1.0B}      & \baseline{3008 (2.6x)}  & \baseline{50 (3.7x)} & ViTAS-F      & 762   & 153 \\
Auto-Ti & 1808   & 183  & DeiT-S  & 1011    & 183  \\
ViTAS-C       & 2712   & 85   & \baseline{OFB-3.6B}     & \baseline{1152 (3.7x)}  & \baseline{252 (3.9x)} \\
DeiT-Ti       & 2613   & 61   & DeiT-B       & 313   & 982 \\
\baseline{OFB-1.1B}      & \baseline{2737 (2.4x)}   & \baseline{57 (3.2x)}    & Auto-B & 357  & 1068  \\
\baseline{OFB-1.7B}     & \baseline{1996 (2.0x)}   & \baseline{81 (2.3x)}    & \baseline{OFB-8.7B}     & \baseline{567 (1.8x)}  & \baseline{741 (1.3x)} \\
\end{tabular}}
\vspace{-1em}
\caption{Throughput and latency results of searched models. `Auto' refers to AutoFormer \cite{autoformer}. The number suffixed with ‘x’ in parentheses denotes the acceleration multiple relative to the original model.}
\label{tab:infer}
\vspace{-1.8em}
\end{table}

\vspace{-1.5em}
\subsection{Transfer Learning Results}
\vspace{-0.5em}
To evaluate the generalization ability of the compressed models by OFB, we further fine-tune the compressed models on downstream datasets, \eg, CIFAR-10, CIFAR-100 \cite{cifar} for image classification, and Cityscape \cite{cityscape} for semantic segmentation. Specifically, we choose DeiT-S \cite{deit} and SETR-DeiT-S \cite{setr} as our baselines. The hyper-parameter setting follows the official fine-tuning strategy in DeiT and SETR \cite{setr}. As presented in Table \ref{deit_cifar}, the compressed models significantly outperform baselines in both accuracy and computation costs, indicating a good generalization ability to downstream datasets for the compressed models by OFB.


\vspace{-0.5em}
\subsection{Ablation Study}
\vspace{-0.5em}
\paragraph{Effectiveness of Weight-sharing Strategy.}
The weight-sharing strategy decides both importance and sparsity distributions, as discussed above. To validate the effectiveness of our bi-mask weight-sharing strategy, we compare the search performance with the ordinal-sharing strategy employed in previous TAS methods and present results in Table \ref{tab:weight_sharing}. From the table, our method obtains more compact models with higher performance than the counterpart, indicating the bi-mask scheme can evaluate unit prunability more accurately. Considering the weight-sharing strategy is only employed at the search stage, we further report the performance gain at the search stage in Table \ref{tab:component}. Compared with the baseline, the bi-mask weight-sharing strategy achieves 1.7\% performance gain, meanwhile reducing 0.14 BFLOPs and 0.6 MParams.

\begin{figure}[t]
\vspace{-2em}
    \centering
    \small
    \subfloat[Learning process of the bi-mask for an MLP layer.]{
    \resizebox{\linewidth}{!}{
    \includegraphics[height=90pt, width=300pt]{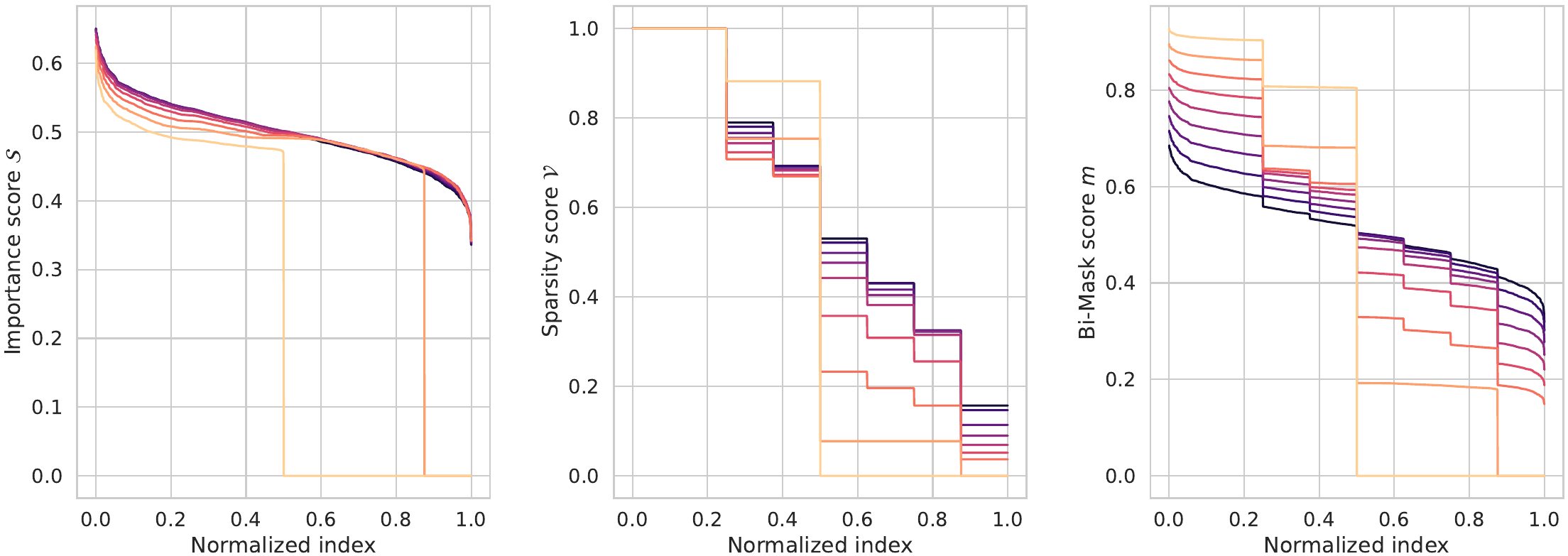}}
    \label{fig:channel_mask}}
    \\
    \subfloat[Learning process of the bi-mask for an MHSA layer.]{
    \resizebox{\linewidth}{!}{
    \includegraphics{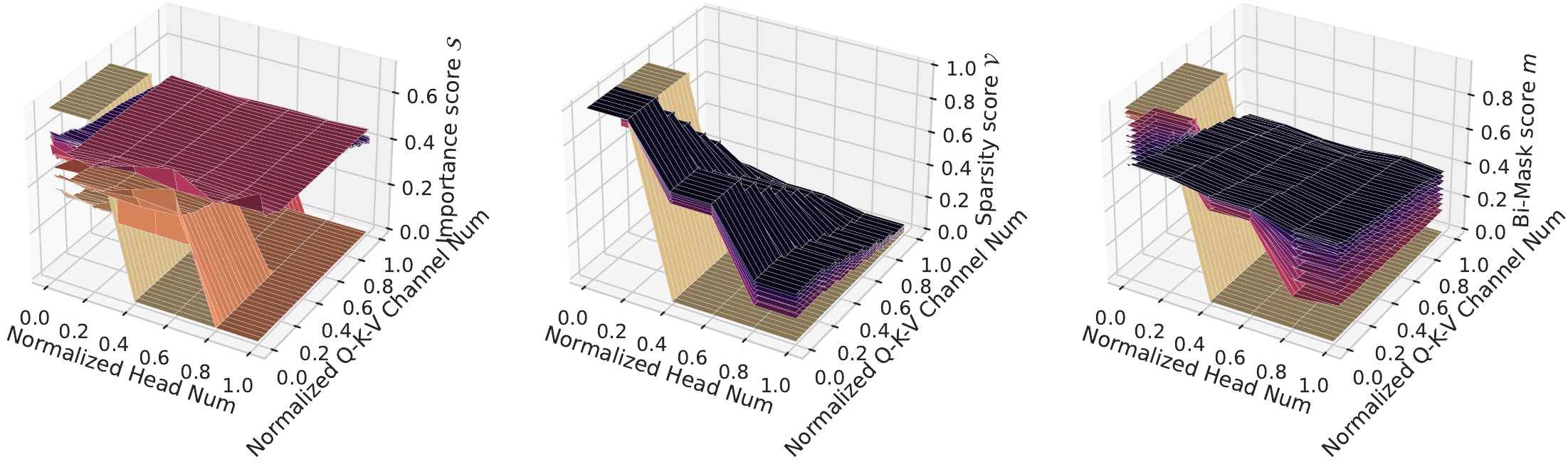}}
    \label{fig:head_mask}}
    \vspace{-1em}
    \caption{Visualization of bi-mask search process. Each line/surface is a descendingly-ordered distribution learned after one-epoch search, with the lighter color denoting a later learned distribution.}
    \label{vis_score}
    \vspace{-2em}
\end{figure}


\paragraph{Effectiveness of Regularization Strategy.}
There are three components for the bi-mask optimization, including the $\ell_1$ regularization $\mathcal{L}_\mathcal{S}$ for importance scores $\mathcal{S}$, the entropy $\mathcal{H}(\alpha)$, and variation regularization $\varPsi(\alpha)$ for sparsity scores $\mathcal{V}$. We test their individual and combinatorial contributions to the search performance, and present results in Table \ref{tab:regularization}. It can be concluded that the entropy regularization $\mathcal{H}(\alpha)$ and the $\ell_1$ regularization $\mathcal{L}_\mathcal{S}$ play the more important role in maintaining high performance, while the variation regularization $\varPsi(\alpha)$ drives the model to approach the target sparsity more closely. (See Appendix \textcolor{red}{B} for theoretical analysis)

\paragraph{Effectiveness of MIM Strategy.}
The MIM effectiveness has been partially demonstrated in Fig. \ref{fig:PMIM}. For clear comparisons, we report search performance with PMIM before retraining in Table \ref{tab:component}. From the table, PMIM can further improve accuracy by 3.6\% Top-1 with a slight increase in computation overhead (+0.05 BFLOPs, +0.2 MParams). In addition, we further analyze the progressive masking strategy in PMIM. As shown in Table \ref{tab:masking_strategy}, we compare PMIM with a constant masking scheme (const.) adopted in \cite{mae, simmim}, which keeps masking 25\% patches as aligned with the final masking ratio in PMIM. It is observed that PMIM improves accuracy at the same FLOPs, thus further validating its effectiveness.

\vspace{-0.5em}
\subsection{Inference Speedup Results}
\vspace{-0.5em}
The throughput and latency results of compressed models are listed in Table \ref{tab:infer}, where the model checkpoints are obtained from official repositories. The GPU throughput is measured on a Tesla V100 GPU with a batch size of 1024, and the latency is measured on Intel(R) Xeon(R) Gold 5218 CPU with one batch size. It shows that OFB achieves superior speedup on different devices. Specifically, the throughput of the compressed DeiT models is accelerated by 1.8x$\sim$3.7x on GPU, while the latency on CPU is reduced by 25\%$\sim$ 74\%.

\vspace{-0.5em}
\subsection{Mask Visualization}
\vspace{-0.5em}
The search process of bi-masks is visualized in Fig. \ref{vis_score} for better understanding, where we take the compressed DeiT-S-1B as an example. It is observed that in the early stage, the bi-mask score (right) is closer to the importance score (left), and the sparsity score distribution changes little, thus the initial learning focuses more on importance evaluation. Whereas, with search going on, the bi-mask score distribution gets progressively closer to the sparsity score distribution, indicating the bi-mask learning focuses more on the sparsity evaluation. As for the pruning, we can find that those units with scores in the small step intervals will be integrated into one step interval, and will be pruned if the sparsity scores drop significantly, which is performed in an adaptive manner.

\begin{figure}[t]
\vspace{-2em}
    \centering
    \small
    \subfloat[MHSAs @1.7BFLOPs.]{
    \begin{minipage}{0.49\linewidth}{\begin{center}
    \resizebox{\linewidth}{!}{\includegraphics{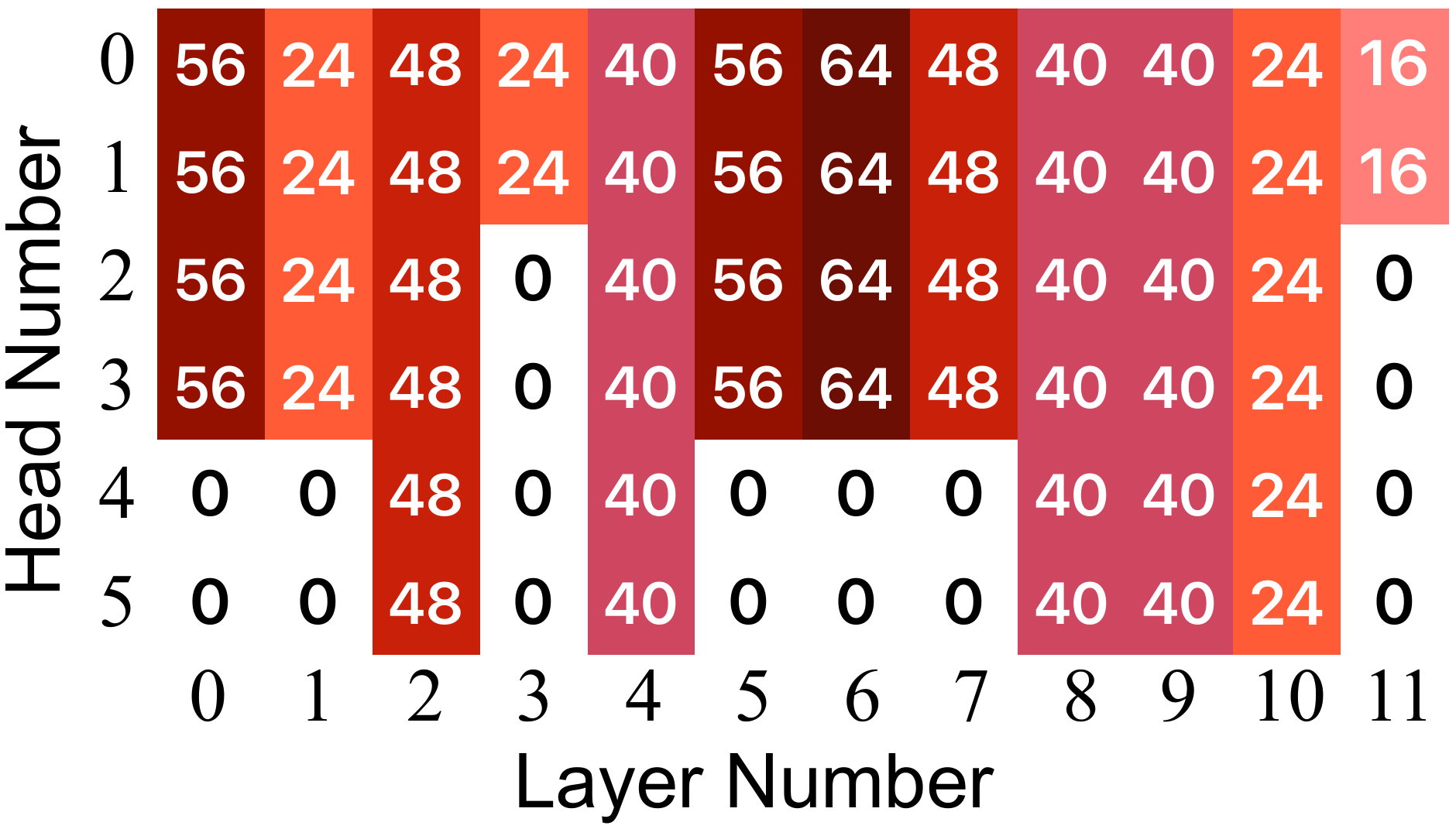}}\end{center}}\end{minipage}
    \label{fig:exp4_attn}}
    \subfloat[MHSAs @1.1BFLOPs.]{
    \begin{minipage}{0.49\linewidth}{\begin{center}
    \resizebox{\linewidth}{!}{\includegraphics{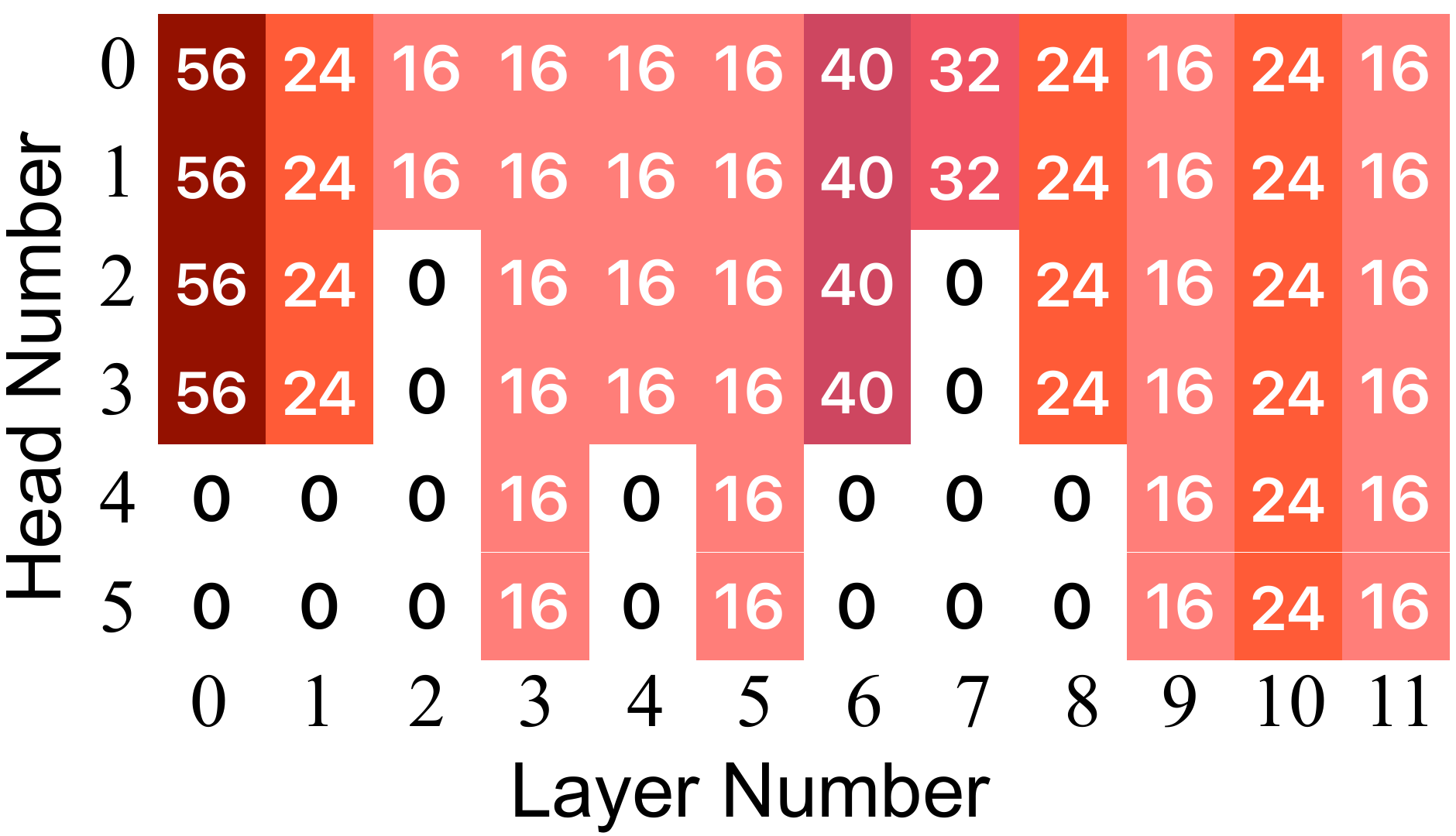}}\end{center}}\end{minipage}
    \label{fig:exp2_attn}}
    \\
    \subfloat[Layerwise MLP dimensions.]{
    \begin{minipage}{0.48\linewidth}{\begin{center}
    \resizebox{\linewidth}{!}{\includegraphics[width=130pt, height=75pt]{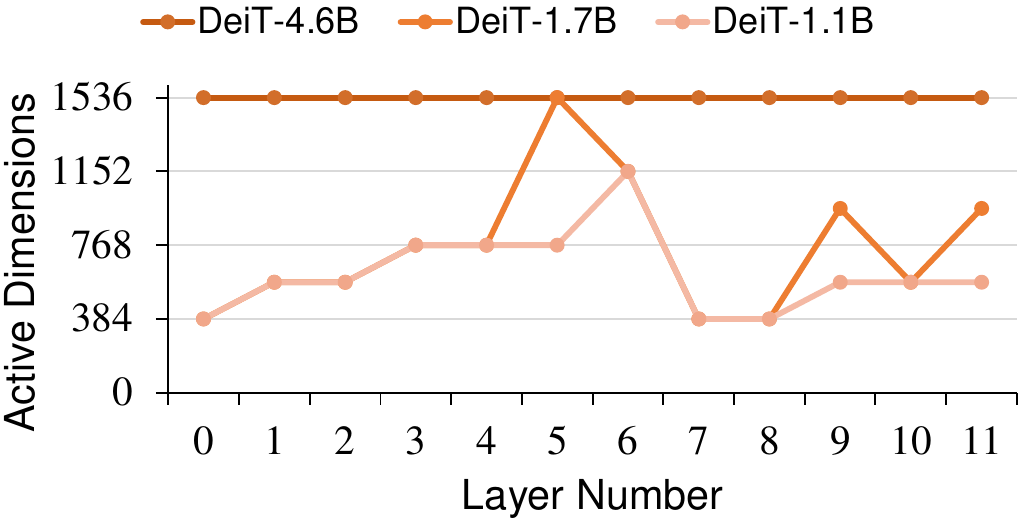}}\end{center}}\end{minipage}
    \label{fig:mlp_vis}
    }
    \subfloat[Patch Embedding dimension.]{
    \begin{minipage}{0.48\linewidth}{\begin{center}
    \resizebox{\linewidth}{!}{\includegraphics[width=130pt, height=75pt]{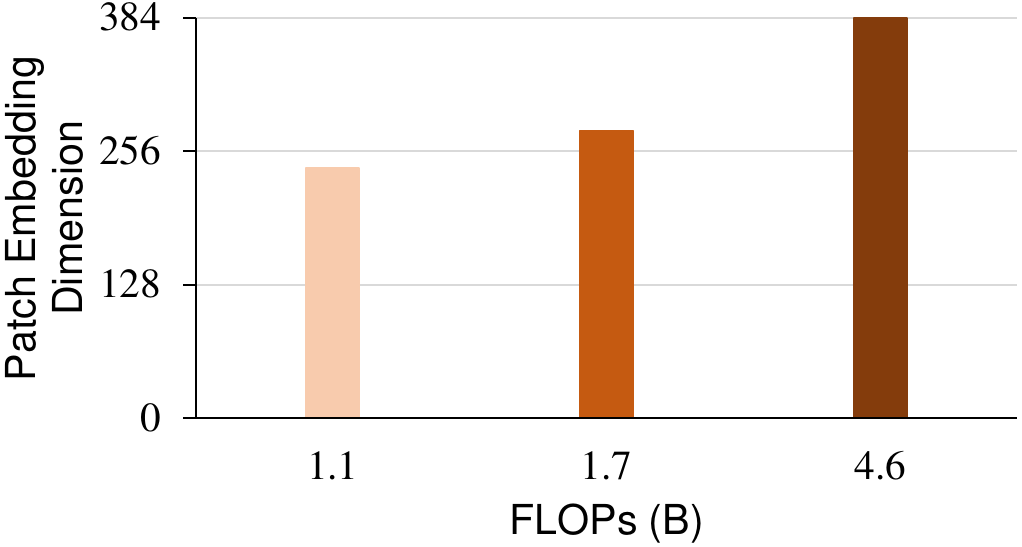}}\end{center}}\end{minipage}
    \label{fig:pe_vis}
    }
    \vspace{-1em}
    \caption{Structure configurations B\&A compression.}
    \label{arch_vis}
\vspace{-2em}
\end{figure}

\subsection{Architecture Visualization}
Fig. \ref{arch_vis} shows several searched architectures of DeiT-S. Fig. \ref{fig:exp4_attn} and \ref{fig:exp2_attn} list the searched results of MHSA modules. There are 12 MHSA layers in DeiT-S, with six heads in each layer and 64 channels in each head. The numbers inside grids denote Q-K-V channels of the head. Here, we jointly assess the prunability of Q-K-C channels for all heads at the same layer to achieve structural acceleration. It is noted that shallow and middle layers consistently need higher Q-K-V dimensions at different budgets, while the head numbers can be reduced.
Fig. \ref{fig:mlp_vis} shows the searched MLP dimensions. Similarly, MLPs are preserved more in middle layers and continually pruned in deep layers, indicating more redundance in deep layers.
Fig. \ref{fig:pe_vis} shows the searched Patch Embedding dimensions, which are carefully searched with more channels preserved. It is reasonable since Patch Embedding is skip-connected to all encoder blocks in DeiT.

\section{Conclusion}
\label{sec:conclusion}
We introduce OFB to tackle the VTC problem. To determine the unit prunability in ViTs, for the first time, OFB explores how to entangle the importance and sparsity scores during search. And a PMIM regularization strategy is specially designed for the dimension-reduced feature space in VTC. Extensive experiments have been conducted to compress various ViTs on ImageNet and downstream benchmarks, indicating an excellent compression capability for ViTs.

\section{Acknowledgements}
The research was supported by National Key R\&D Program of China (Grant No. 2022ZD0160104), National Natural Science Foundation of China (No. 62071127, and
62101137), National Key Research and Development Program of China (No. 2022ZD0160100), Shanghai Natural Science Foundation (No. 23ZR1402900), the Science and Technology Commission of Shanghai Municipality (Grant No. 22DZ1100102), and Shanghai Rising Star Program (Grant No. 23QD1401000).
{
    \small
    \bibliographystyle{ieeenat_fullname}
    \bibliography{main}
}

\newpage
\setcounter{page}{1}
\setcounter{equation}{0}
\setcounter{figure}{0}
\setcounter{table}{0}
\appendix

\maketitlesupplementary
\renewcommand\thefigure{A.\arabic{figure}}
\renewcommand\theequation{A.\arabic{equation}}
\renewcommand\thetable{A.\arabic{table}}

In supplementary material, we first provide a summary of all notations mentioned in the main body for a clear understanding of the paper, as shown in Table \ref{tab:notation}. Then, we make a deep analysis of the adaptive one-hot loss function, including the theoretical justification of its effectiveness (Appendix \ref{e_and_v}). In addition, we demonstrate the necessity of activating variance regularization with the tangent function from the lens of optimization space (Appendix \ref{tangent}). Then, we introduce the implementation details of our experiments on different baseline models (Appendix \ref{details}). Finally, we provide more experimental analyses (Appendix \ref{exp_vis}, \ref{inductive}, and \ref{cnn}). Our code is available at \url{https://github.com/HankYe/Once-for-Both}.

\begin{table*}[t]
\centering
\small
\begin{tabular}{x{80}y{140}||x{80}y{140}}
Symbol    & Notation & Symbol           & Notation \\ \hline
$\mathcal{N}$       & supernet                & $m$                         & bi-mask \\
$\mathcal{A}$       & search space            & $\lambda(t)$                & time-varying weight coefficient \\
$W$                 & supernet weights        & $i$                         & the submodule index \\
$\mathcal{F}_d$     & decoder                 & $j$                         & the unit index \\
$\mathcal{D}$       & dataset                 & $\alpha$                    & architecture parameters \\
$\boldsymbol{f}$    & importance criterion    & $p$                         & the normalized architecture score \\
$\mathcal{S}$       & importance score        & $\Delta$                    & search step \\
$\mathcal{V}$       & sparsity score          & $M$                         & submodule number in $\mathcal{N}$ \\
$\mathcal{L}_{val}$ & validation loss         & $\sigma$                    & measured variance of $p$ \\
$\mathcal{L}_{train}$& training loss          & $\sigma^{t}$                & target variance of $p$ \\
$\mathcal{L}_{\mathcal{S}}$ & regularization item for $\mathcal{S}$ & $\omega$ & normalized variance of $p$ \\
$\mathcal{L}_{\mathcal{V}}$ & regularization item for $\mathcal{V}$ & $\mu_1, \mu_2$ & weight coefficients in $\mathcal{L}_{\mathcal{V}}$ \\
$\mathcal{L}_{rec}$  & reconstruction loss    & $\mu_3$                & weight coefficient in $\mathcal{L}_{\mathcal{S}}$\\
$\mathcal{L}_{m}$    & regularization item for $m$ & $\bar{p}$               & mean of $p$ \\
$R(p)$               & regularization item for $p$ & $\eta$                      & scaling factor  \\
$g$                  & computation cost       & $\gamma$                    & masking ratio \\
$\tau$               & resource constraint    & $\Delta\mathrm{T}$          & pruning interval \\
\end{tabular}
\caption{Notation Summary.}
\label{tab:notation}
\end{table*}

\section{Motivation behind Entropy and Variance Regularizations}\label{e_and_v}
In this section, we take a deep dive into the design of the adaptive one-hot loss function, which targets discretizing each $p_i$ into a one-hot vector. Considering the inaccessible one-hot index in the search process (Sec. \textcolor{red}{3.3}), the optimization objective of $p_i$ should adapt to potential members in a one-hot vector group according to the search results of other submodules. For example, the sparsity target of one submodule can change from [0, 0, 1, 0] to [0, 1, 0, 0] if other pruned submodules contribute to a small computation reduction. Our method employs an adaptive one-hot loss function to learn members' invariant and unique properties within a one-hot vector group, fulfilling the optimization objectives.

First, we present two learnable properties of the one-hot vector set: entropy $\mathcal{H}(p_i)$ and variance $\sigma(p_i)$. Focusing on the $i$-th dimension in the normalized $\alpha$ using \textit{softmax}, denoted as $p_i$, we establish a theorem revealing the unique relationship between $\mathcal{H}(p_i)$, $\sigma(p_i)$, and the set of one-hot vectors.
\newtheorem{thm}{\bf Theorem}
\begin{thm}\label{thm1}

\small
Suppose $p_i\in R^{1\times D}$ and $\sum_{k=1}^{D}p_{ik}=1$, with $p_{ik} \geq 0, k=1, 2,..., D$. Then the following propositions are equivalent:
\begin{flalign}
\label{entropy}&\left( 1 \right) \,\,\mathcal{H} (p_i)=-\sum\nolimits_{k=1}^D{p_{ik}\log p_{ik}}=0;
&\\
\label{variance}&\left( 2 \right) \,\,\sigma \left( p_i \right) =\sum\nolimits_{k=1}^D{\left( p_{ik}-\bar{p} \right) ^2}/D=\left( D-1 \right) /D^2;
& \\
\label{one-hot}&\left( 3 \right) \,\,p_i\in \left\{ e_k \right\} , k=1,2,...,D,&
\end{flalign}
where $e_k$ represents the $D$-dimensional one-hot vector with the $k$-th element set to one.
\end{thm}

\begin{proof}
We prove the equivalence by demonstrating \ref{entropy} $\Leftrightarrow$ \ref{one-hot} and \ref{variance} $\Leftrightarrow$ \ref{one-hot}.

As for the former equivalence, given that $\small p_i\in \left\{ e_k \right\}$, the entropy of $p_i$ can be easily computed as zero, thus \ref{one-hot} $\Rightarrow$ \ref{entropy}. Then we prove that \ref{entropy} $\Rightarrow$ \ref{one-hot}. Since $p_i$ is constrained by {\small $\sum_{k=1}^{D}p_{ik}=1$}, we construct a Lagrange function $L(p_i, \lambda)$ as follows:
\begin{equation}
\small
    L(p_i, \lambda) = \mathcal{H}(p_i) + \lambda (1-\sum\nolimits_{k=1}^{D}p_{ik}),
\end{equation}
where $\lambda$ is the Lagrange multiplier. Now, we analyze the extremum and monotonicity of $L(p_i, \lambda)$ by taking partial derivatives with respect to $p_{ik}$ and $\lambda$, as shown in Eq. (\ref{partial_L}).
\begin{equation}\label{partial_L}
\small
    \frac{\partial L}{\partial p_{ik}}=-1-\lambda -\log p_{ik}, \,\,\,\,k=1,2,...,D.
\end{equation}

By setting each partial derivative to zero, we can obtain that:
$\lambda=\log{D}-1, p_{ik}=D^{-1}, k=1,2,...,D$. If $0<p_{ik}<D^{-1}$, then $\partial L/\partial p_{ik}=-\log(Dp_{ik}) >0, k=1,2,...,D$. Similarly, if $D^{-1}<p_{ik}<1$, then $\partial L/\partial p_{ik}=-\log(Dp_{ik}) <0, k=1,2,...,D$. Consequently, $\mathcal{H}(p_i)$ is monotone increasing if $0<p_{ik}<D^{-1}$ and monotone decreasing if $D^{-1}<p_{ik}<1$ in the dimension of $p_{ik}$. Therefore, $\mathcal{H}(p_i)$ reaches maximum as $\log D$ when $p_i=D^{-1}\mathbf{1}_{1\times D}$, and reaches the minimum as zero when $p_{ik}\in \{0,1\}, k=1,2,..,D$. In other words, $\mathcal{H}(p_i)=0$ holds only if $p_i\in \{e_k\}, k=1,2,...,D$. Therefore, $\mathcal{H}(p_i)=0 \Rightarrow p_i\in \{e_k\}$, \ie, \ref{entropy} $\Rightarrow$ \ref{one-hot}. Finally, we can get the conclusion that \ref{entropy} is equivalent to \ref{one-hot}.

As for the latter equivalence, \ref{variance} $\Leftrightarrow$ \ref{one-hot}, the proof is similar and thus omitted here.

\end{proof}
Based on the analysis, the entropy and variance regularization of $p_i$ can effectively drive it towards a one-hot vector, discretizing both $p_{ik}$ and the sparsity score $\mathcal{V}$ as binary.
 
Having demonstrated the effectiveness of both regularization items, another question is why both regularization items should be employed. To answer this question, we analyze the respective contribution of the entropy and variance regularization items to the discretization process, making the following observations.

\begin{thm}\label{thm2}
\small
Let $\|\mathcal{H}(p_i)-0\|$ and $\|\sigma(p_i)-(D-1)/D^2\|$ denote the regularization items for entropy and variance, respectively. Then, the following properties hold:\\
$\left( 1 \right) \,\,\|\mathcal{H}(p_i)-0\|$ works as $\ell_1$ sparsity for $p_i$, guiding $p_i$ towards a potential one-hot vector;\\
$\left( 2 \right) \,\,\|\sigma(p_i)-(D-1)/D^2\|$ works as $\ell_2$ smoothness for $p_i$, facilitating a seamless transition from one potential target vector to another.
\end{thm}

\begin{figure*}[t]
\begin{minipage}{0.3\linewidth}
    \centering
    \small
    \resizebox{\linewidth}{!}{\includegraphics{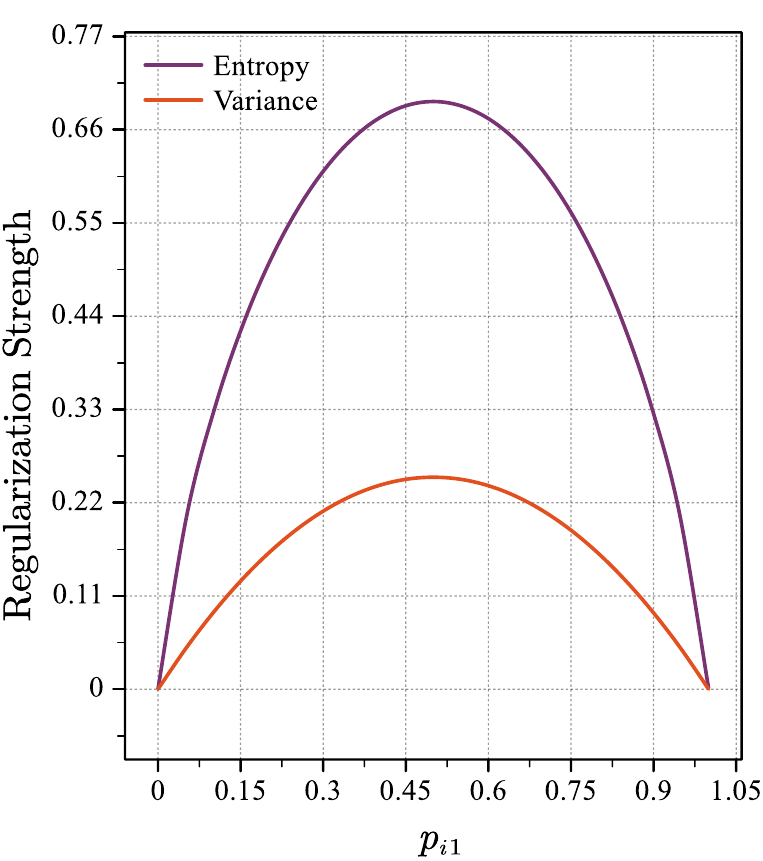}}
    \vspace{-7pt}
    \caption{Visualization of entropy and variance regularization distributions under the two-dimension $p_i$ setting for simplicity.\\}
    \label{reg_vis}
\end{minipage}
\hfill
    \begin{minipage}{0.3\linewidth}
        \centering
    \small
    \resizebox{\linewidth}{!}{
    \begin{tikzpicture}
    \node[anchor=south west,inner sep=0] (image) at (0,0) {\includegraphics{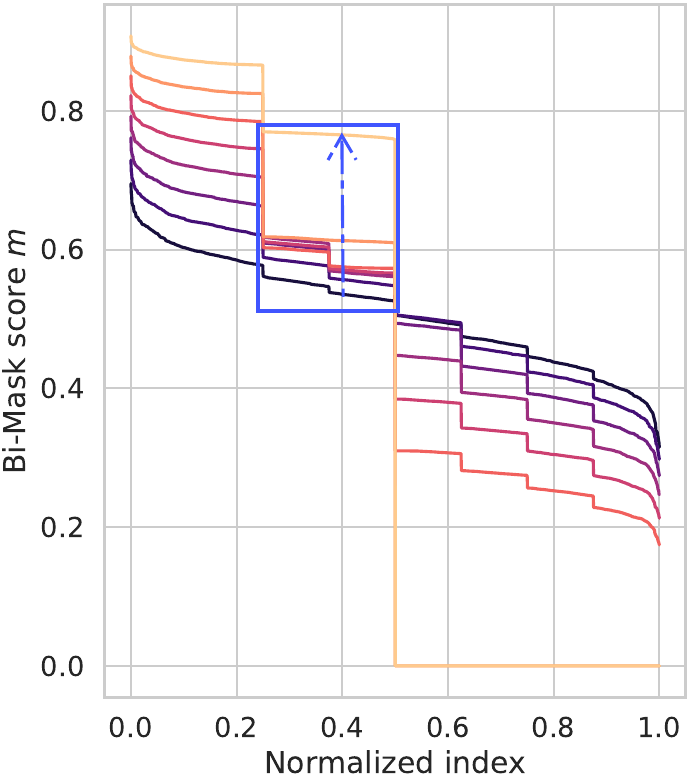}};
    \begin{scope}[x={(image.south east)},y={(image.north west)}]
    \node[fill=none, draw=blue, text=black, rounded corners, minimum width=2.4cm, minimum height=3.5cm, line width=2.5pt] at (0.475,0.71){};
    \draw[-{Stealth[length=15pt,width=8pt]}, dashed, line width=2pt] (0.5,0.62) -- (0.5,0.83);
    \end{scope}
    \end{tikzpicture}}
    \vspace{-7pt}
    \caption{Visualization of the learning process of one bi-mask sampled from the compression of DeiT-S with the target sparsity of 1BFLOPs and without variance regularization.}
    \label{no_var_vis}
    \end{minipage}
    \hfill
    \begin{minipage}{0.3\linewidth}
        \centering
    \small
    \resizebox{\linewidth}{!}{\begin{tikzpicture}
    \node[anchor=south west,inner sep=0] (image) at (0,0) {\includegraphics{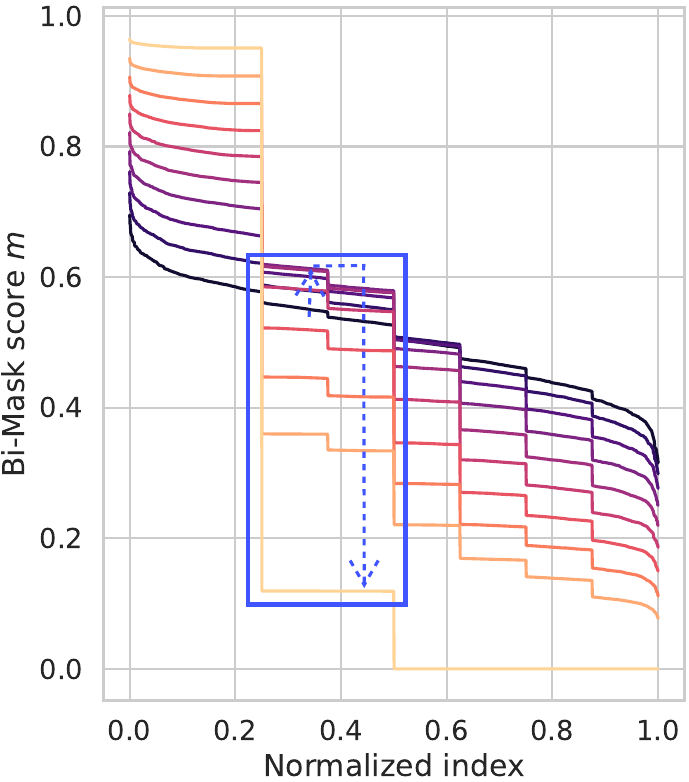}};
    \begin{scope}[x={(image.south east)},y={(image.north west)}]
    \node[fill=none, draw=blue, text=black, rounded corners, minimum width=2.4cm, minimum height=6cm, line width=2pt] at (0.475,0.45){};
    \draw[-{Stealth[length=15pt,width=8pt]}, dashed, line width=2pt] (0.45,0.6) -- (0.45,0.66);
    \draw[-{[length=15pt,width=8pt]}, dashed, line width=2pt] (0.45,0.66) -- (0.53,0.66);
    \draw[-{Stealth[length=15pt,width=8pt]}, dashed, line width=2pt] (0.53,0.66) -- (0.53,0.24);
    \end{scope}
    \end{tikzpicture}}
    \vspace{-7pt}
    \caption{Visualization of the learning process of one bi-mask sampled from the compression of DeiT-S with the target sparsity of 1BFLOPs and without entropy regularization.}
    \label{no_entr_vis}
    \end{minipage}
    \vspace{-1em}
\end{figure*}

\begin{proof}
We first derive the approximate order of two items to identify regularization types. Then we analyze the function of each regularization from the lens of the optimization space.

As for entropy regularization, since {\small $\mathcal{H}(p_i)\geq 0$}, the regularization can be simplified as {\small $\mathcal{H}(p_i)$}. Further, according to \cite{renyi_entropy}, $\mathcal{H}(p_i)$ can be regarded as the first-order entropy of the distribution $p_i$, as shown in Eq. (\ref{order_1_H}).
\begin{equation}\label{order_1_H}
\small
    \mathcal{H} \left( p_i \right) =\lim_{r\rightarrow 1} \mathcal{H} _r\left( p_i \right) =\lim_{r\rightarrow 1} \frac{1}{1-r}\log \left( \sum_{k=1}^D{p_{ik}^{r}} \right),
\end{equation}
where $\mathcal{H} _r\left( p_i \right)$ is the generalized entropy measure, Rényi Entropy. Therefore, $\|\mathcal{H}(p_i)-0\|$ can be regarded as $\ell_1$ sparsity for the discrete distribution of $p_i$.

As for variance regularization, since {\small $ \sigma(p_i)\leq(D-1)/D^2$}, the regularization can be simplified as {\small $\small (D-1)/D^2 - \sigma(p_i)$}. Further, we expand {\small $\sigma(p_i)$} into the polynomial form as follows:
\begin{equation}\label{var_l2}
\small
\begin{aligned}
    \frac{D-1}{D^2}-\sigma (p_i)&=\frac{D-1}{D^2}-\frac{\sum\nolimits_{k=1}^D{p_{ik}^{2}}-\frac{2}{D}\sum\nolimits_{k=1}^D{p_{ik}}+\frac{1}{D}}{D}
\\
&=\frac{D-1}{D^2}-\frac{\sum\nolimits_{k=1}^D{p_{ik}^{2}}-\frac{1}{D}}{D}
\\
&=\frac{1}{D}\left( 1-\sum\nolimits_{k=1}^D{p_{ik}^{2}} \right).
\end{aligned}
\end{equation}

From Eq. (\ref{var_l2}), minimizing {\small $(D-1)/D^2 - \sigma(p_i)$} can be regarded as maximizing $\ell_2$-norm of $p_i$. Therefore, the variance regularization cannot be viewed as the $\ell_2$ sparsity. Instead, compared with entropy regularization, we argue that variance regularization works as $\ell_2$ smoothness.

Specifically, we visualize the distributions of entropy and variance regularization in Fig. \ref{reg_vis}. Considering the normalization constraint, we focus on the two-dimensional setting of $p_i$ to simplify the analysis. From the figure, we observe that variance regularization distribution is flatter than entropy regularization in the region neighboring the maximum point. With the increase in the dimensionality of $p_i$, the entropy regularization becomes stronger (sharper peak), while the variance regularization becomes weaker (flatter peak). Consequently, the variance regularization effectively smooths the optimization space.
\end{proof}

Note that during optimization, entropy regularization is sensitive to the initialization of $p_i$, as the gradient continually drives the maximum $p_{ik}$ towards one. This behavior is independent of the sparsity constraint and leads to fixing the potential one-hot index throughout the pruning process. This issue is evident in the results of lines 1 and 5 in Table \textcolor{red}{5d} and Fig. \ref{no_var_vis}. In lines 1 and 5 of Table \textcolor{red}{5d}, the searched models have the same size and are the largest among all searched models. The bi-mask score learning process in Fig. \ref{no_var_vis}, which is sampled from one submodule in DeiT-S compressed by entropy regularization alone, shows that the bi-mask scores in different segments are continually increased or decreased. Therefore, \textbf{entropy regularization mainly contributes to the score sparsity but is constrained by the initial score distribution}.

In contrast, variance regularization is agnostic to the one-hot index and operates within a much flatter optimization space than entropy regularization. This characteristic allows variance regularization to adaptively adjust the target one-hot vector based on the search results of other submodules or units. Fig. \ref{no_entr_vis} visualizes the bi-mask score learning process from the same submodule as in Fig. \ref{no_var_vis}, employing the same compression target. The scores in the box initially increase and gradually decrease after the pruning of other units, indicating a switch in the target one-hot vector from [0, 0, 1, 0, 0, 0, 0] to [1, 0, 0, 0, 0, 0, 0]. Additionally, the results in lines 2 and 4 of Table \textcolor{red}{5d} demonstrate that models compressed solely with variance regularization achieve the smallest model size (4.0MParams, 0.8BFLOPs) while satisfying the sparsity constraint. Hence, \textbf{variance regularization primarily contributes to a smoother optimization space, enabling easier adjustment of the target one-hot vector to align with the pruning process of other units and the desired sparsity constraint.}

In summary, we demonstrate the effectiveness and necessity of both entropy and variance regularization items from the lens of equivalent properties, regularization types, and optimization contributions.

\section{Motivation behind Tangent Activation for Variance Regularization}\label{tangent}
As mentioned earlier, applying $\ell_1$ regularization to the discrete variable $p_i$ and continuous variable $\mathcal{S}$ enhances model sparsity while maintaining model performance. This observation is also supported by the results in Table \textcolor{red}{5d} (lines 1, 5, 6, and 7).

\begin{figure}[t]
    \centering
    \small
    \resizebox{0.6\linewidth}{!}{\includegraphics{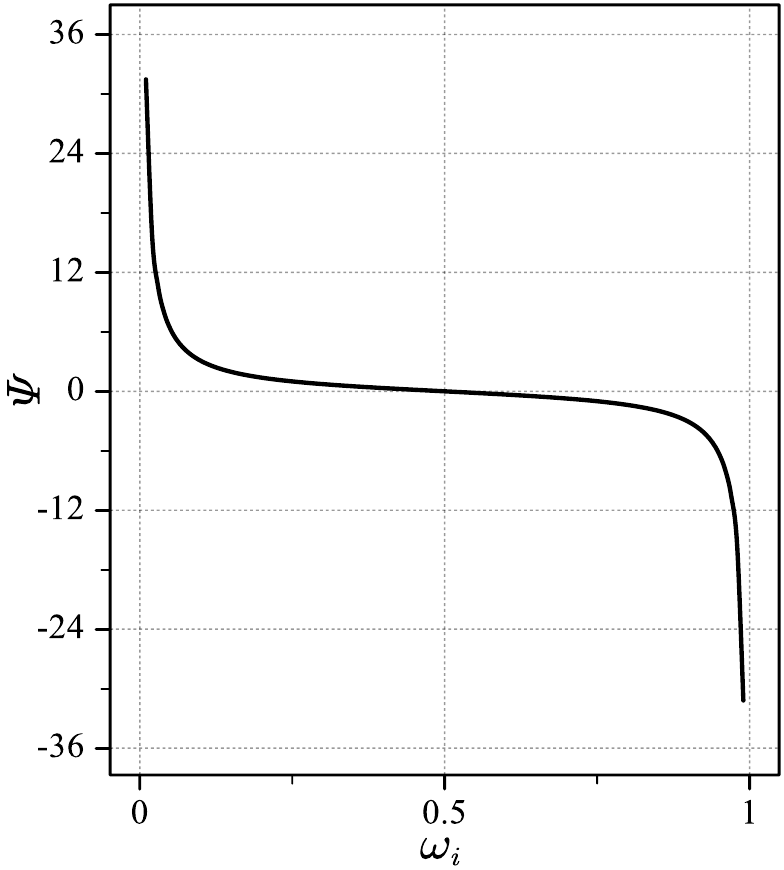}}
    \vspace{-7pt}
    \caption{Distribution of tangent-activated variance regularization.}
    \label{tan_vis}
\vspace{-25pt}
\end{figure}

Now, let's delve into the explanation for why we use tangent to map the variance regularization. This choice is primarily motivated by the over-smoothness present in the high-dimensional optimization space during the early search process. To better understand this, we analyze the regularization strength in the high-dimensional optimization space. In particular, we focus on scenarios where the dimensionality $D$ is much larger than 2 ({\small $D >> 2$}). In such cases, the maximum strength of {\small $(D-1)/D^2 - \sigma(p_i)$}, which equals {\small $(D-1)/D^2$}, is very close to the minimum strength of zero. As a result, the gradient of the variance regularization becomes extremely small compared to the gradient of the entropy regularization. Consequently, the impact of variance regularization on searching for the potential target one-hot vector becomes minimal, as the gradient of variance regularization is overwhelmed by that of entropy regularization.

\begin{figure*}[t]
    \centering
    \resizebox{\linewidth}{!}{
    \includegraphics{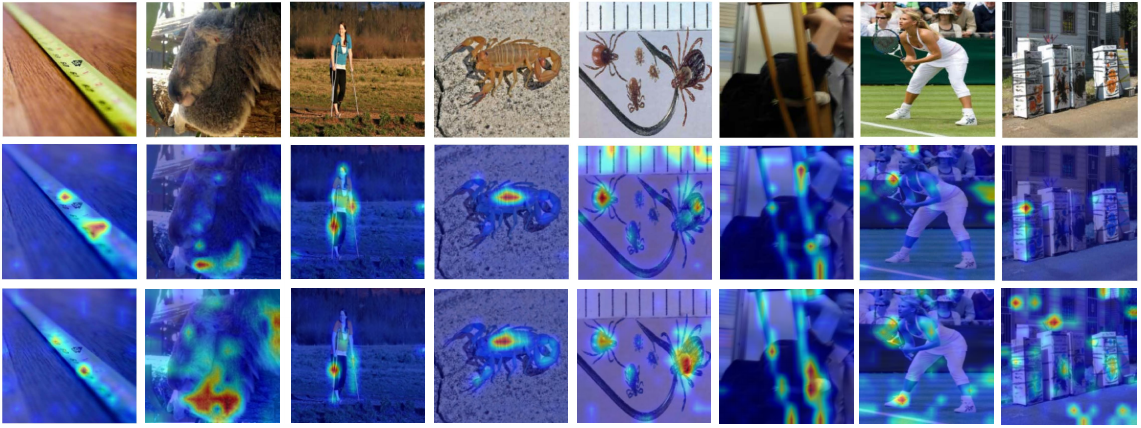}}
    \caption{Attention maps from different models for several images sampled from ImageNet-1K. The first row is the original images, the second row represents visualizations from DeiT-S, and the third row denotes the results from DeiT-B\_3.6BFLOPs.}
    \label{fig:attn}
\end{figure*}

To solve this problem, we propose utilizing tangent activation to produce a large gradient during the search for the potential target one-hot vector under the performance and sparsity objectives. Specifically, we design the activation as follows:
\begin{equation}
\small
    \varPsi(p_i) = \tan \left( \frac{\pi}{2}-\pi \omega _i \right),
\end{equation}
where {\small $\omega _i=\sigma _{i}/\sigma _{i}^{t} \in [0, 1]$} as mentioned in the main body. The distribution of {\small $\varPsi(p_i)$} is presented in Fig. \ref{tan_vis}, where {\small $\varPsi$} rapidly decreases when {\small $\omega_i$} is close to zero, \ie, $\sigma_i$ is close to zero. In other words, if the variance of $p_i$ is small, referring to the initial distribution of $p_i$ or the distribution after pruning small-score units, we will assign a larger gradient for variance regularization than entropy one. This prioritization allows for faster optimization of $p_i$ towards the potential target one-hot vector under the performance and sparsity objective. By doing so, we prevent entropy regularization from dominating the optimization process and ensure the target one-hot vector can be dynamically adjusted.

Once the potential target one-hot vector is found, the optimization process should prioritize entropy regularization. This is because entropy regularization, as an approximate $\ell_1$ sparsity measure, can promote sparsity in $p_i$ while maintaining model performance. Therefore, the gradient of variance regularization can be suppressed to minimize interference from other one-hot objectives. When the distribution of $p_i$ is close to a one-hot vector, meaning $\omega_i$ is close to one, the significant gradient of $\varPsi$ can facilitate the discretization of $p_i$ in the same optimization direction as entropy regularization. In this situation, the disturbance caused by variance regularization from other one-hot objectives is typically negligible.

Based on the above analysis, the main contribution of tangent activation is providing a large gradient to adjust the potential target one-hot vector of $p_i$ that satisfies the compression requirement every time the small-score units in $p_i$ are pruned. Therefore, as validated in lines 1, 6, 5, and 7 of Table \textcolor{red}{5d}, the variance regularization {\small $\varPsi(p_i)$} can drive the model to approach the target sparsity more closely and more efficiently.

\section{Implementation Details}\label{details}
OFB adopts the searching-and-retraining scheme as previous works do. All experiments are conducted with 8 V100 GPUs. In the search process, we use the pre-trained models released from official implementation on ImageNet-1K as the supernet $\mathcal{N}$. The decoder $\mathcal{F}_d$ consists of one convolutional layer and a pixel-shuffle layer as SimMIM \cite{simmim} does. We search for 100 epochs on DeiT-S and Swin-Ti, and 200 on DeiT-B, with 20 epochs for warming up. The other learning schedules and the augmentation strategy follow the official settings in the respective papers. The learning schedules of {\small $\alpha$}, {\small $\mathcal{S}$} and {\small $\mathcal{F}_d$} shares the same setting as {\small $W$}, except that {\small $\beta_1$} is set as 0.5 for the optimizer of {\small $\{\alpha, \mathcal{S}\}$}. The default values of $\mu_1, \mu_2, \mu_3, \mathrm{and\,} \eta$ are set as 5e-1, 1e2, 2e-5, and 2e-1, respectively. The unit pruning is initiated at every one-third interval ($\Delta\mathrm{T}$) within each epoch. In the retraining process, we follow the default training strategy reported in official papers except for mixup \cite{mixup} and cutmix \cite{cutmix}, which impair the retraining performance in our setting, and the learning rate is set as 6e-4 for both types of models. The masking ratio linearly increases from 1\% to 25\% of the input patches for DeiTs and that of the downsampled patches for Swin-Ti during the search stage. 
\section{Additional Attention Map Visualizations}\label{exp_vis}
We take DeiT-B\_3.6BFLOPs as an example to compare the attention maps with DeiT-S, which shares the same depth as DeiT-S and has higher performance but with smaller FLOPs and parameters. We adopt the method introduced in \cite{te} to visualize the attention maps from the output layer. The results are shown in Fig. \ref{fig:attn}. From the figure, it can be observed that the compressed model focuses more on the extraction of class-specific contextual information, meanwhile suppresses some useless information, \eg, the background features in the picture of the fifth column. This indicates that OFB can effectively evaluate the prunability of units in different submodules and finally preserve useful and important units to perform high compression performance.

\begin{table}[t]
\centering
\small
\begin{tabular}{lccc}
Case (w/o rt.) & Top-1 (\%) & FLOPs (B)& \#Param (M)\\
\shline
Uniform init & 63.5 & 0.6 & 3.0\\
Random init & \baseline{\textbf{72.8}} & \baseline{\textbf{1.1}} & \baseline{\textbf{5.3}}\\
\end{tabular}
\caption{Inductive bias analysis.}
\label{tab:inductive}
\end{table}

\section{Inductive Bias Analysis}\label{inductive}
We further explore the impact of inductive bias on the search performance. As shown in Table. \ref{tab:inductive}, with the same computation constraint, despite the smaller model size, the uniformly-initialized search space performs poorly in model performance, while the randomly-initialized one can achieve better tradeoff between model performance and compression budget, demonstrating the negative impact of inductive bias in the initialization of model parameters.

\section{Generalization Ability on CNNs}\label{cnn}
To test the generalization ability of OFB, we apply it to compressing ResNet-50 on ImageNet. As shown in Table \ref{tab:res50}, compared with baseline and SOTA models, OFB achieves comparable performance with higher compression ratio, which further demonstrates the superiority of OFB in generalization ability.
\begin{table}[t]
\centering
\small
\begin{tabular}{lccc}
Model &  Top-1 (\%) &  Top-5 (\%) & FLOPs (B)\\
\shline
ResNet50 & 76.2 & 92.9 & 4.1 \\
DepGraph \cite{Fang_2023_CVPR} & 75.8 & - & 2.0\\
OFB & \baseline{\textbf{75.8}} & \baseline{\textbf{92.6}} & \baseline{\textbf{1.6}}\\
\end{tabular}
\caption{Generalization Performance on ResNet-50.}
\label{tab:res50}
\end{table}


\end{document}